\documentclass{article}

\usepackage{microtype}
\usepackage{graphicx}
\usepackage{subfigure}
\usepackage{booktabs} 

\usepackage{hyperref}
\usepackage{url}

\usepackage[accepted]{icml2024}

\usepackage{amsmath}
\usepackage{amssymb}
\usepackage{mathtools}
\usepackage{bbding}
\usepackage{amsthm}
\usepackage{braket}
\usepackage{enumitem}

\usepackage[capitalize,noabbrev]{cleveref}

\theoremstyle{plain}
\newtheorem{theorem}{Theorem}[section]
\newtheorem{claim}[theorem]{Claim}
\newtheorem{proposition}[theorem]{Proposition}

\theoremstyle{definition}

\theoremstyle{remark}

\usepackage{multirow}
\usepackage{tikz}
\usepackage{pgfplots}
\usepackage{pgfplotstable}
\usepackage{makecell}

\usepackage[textsize=tiny]{todonotes}

\icmltitlerunning{Generalizing Knowledge Graph Embedding with Universal Orthogonal Parameterization}

\begin{document}

\twocolumn[
\icmltitle{Generalizing Knowledge Graph Embedding with \\ Universal Orthogonal Parameterization}




\begin{icmlauthorlist}
\icmlauthor{Rui Li}{gsai}
\icmlauthor{Chaozhuo Li}{by}
\icmlauthor{Yanming Shen}{dlut}
\icmlauthor{Zeyu Zhang}{gsai}
\icmlauthor{Xu Chen}{gsai}
\end{icmlauthorlist}

\icmlaffiliation{gsai}{Gaoling School of Artificial Intelligence, Renmin University of China, Beijing, China}
\icmlaffiliation{by}{Key Laboratory of Trustworthy Distributed Computing and Service (MOE), Beijing University of Posts and Telecommunications, China}
\icmlaffiliation{dlut}{School of Computer Science and Technology, Dalian University of Technology, China}

\icmlcorrespondingauthor{Xu Chen}{xu.chen@ruc.edu.cn}

\icmlkeywords{Machine Learning, ICML}

\vskip 0.3in
]



\printAffiliationsAndNotice{}  

\begin{abstract}
Recent advances in knowledge graph embedding (KGE) rely on Euclidean/hyperbolic orthogonal relation transformations to model intrinsic logical patterns and topological structures.
However, existing approaches are confined to rigid relational orthogonalization with \emph{restricted dimension} and \emph{homogeneous geometry}, leading to deficient modeling capability.
In this work, we move beyond these approaches in terms of both dimension and geometry by introducing a powerful framework named GoldE, which features a \emph{universal} orthogonal parameterization based on a generalized form of Householder reflection.
Such parameterization can naturally achieve \emph{dimensional extension} and \emph{geometric unification} with theoretical guarantees,
enabling our framework to simultaneously capture crucial logical patterns and inherent topological heterogeneity of knowledge graphs.
Empirically, GoldE achieves state-of-the-art performance on three standard benchmarks.
Codes are available at 
\url{https://github.com/xxrep/GoldE}.
\end{abstract}

\section{Introduction}
Knowledge graphs (KGs) structurally organize vast human knowledge into the factual triples, where each triple $(h,r,t)$ signifies the existence of a relation $r$ between head entity $h$ and tail entity $t$.
Imbued with a wealth of factual knowledge, KGs have facilitated a myriad of downstream applications ~\cite{app2, app3, thinkkg}.
However, real-world KGs such as Freebase~\cite{bollacker2008freebase} are plagued by incompleteness~\cite{TransE}.
This motivates substantial research on Knowledge Graph Embedding (KGE), which learns expressive representations of entities and relations for predicting the missing links.

The key to KGE lies in capturing crucial \emph{logical patterns} (e.g., symmetry, antisymmetry, inversion and composition) \cite{sun2019rotate} and inherent \emph{topological structures} (e.g., cyclicity and hierarchy) \cite{RotH} as illustrated in Figure~\ref{Fig-1}.
Recent advances highlight the effectiveness of
\emph{orthogonal relation transformations} for this principle.
As a pioneering work, RotatE~\cite{sun2019rotate} represents relations as $2$-dimensional Euclidean rotations, i.e., isometries on $1$-spheres~\cite{sommer2013geometric}, which is able to model the four logical patterns and cyclical structures.
The subsequent works can be categorized into two branches:
(1) ``\emph{Dimension Branch}''~\cite{quate, rotate3d, duale, house} extends the (Euclidean) orthogonal relation transformations into higher-dimensional spaces for better modeling capacity; 
(2) ``\emph{Geometry Branch}''~\cite{MuRP, RotH} leverages hyperbolic isometries to preserve the hierarchical structures.


\begin{figure}[t!]
\centering
\includegraphics[width=0.882\columnwidth]{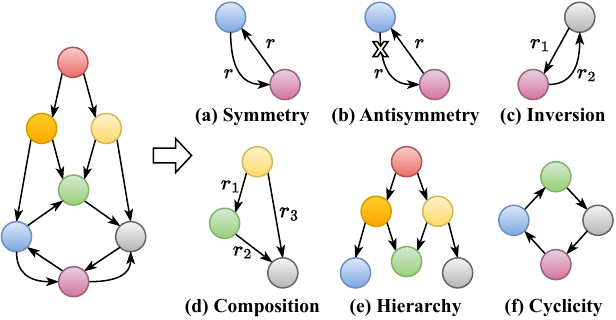}
\vspace{-3.0mm}
\caption{Illustrations of logical patterns (a-d)~\cite{sun2019rotate} and topological structures (e-f)~\cite{RotH}.}
\label{Fig-1}
\vspace{-4.0mm}
\end{figure}

\begin{table*}[t]
\caption{Recent models' capability of modeling logical patterns and topological structures of KGs. $\mathbb{E}$, $\mathbb{P}$ and $\mathbb{Q}$ denote Euclidean, elliptic and hyperbolic geometries for orthogonal relation transformations, respectively. Note that Euclidean models essentially act by isometries on hyperspheres~\cite{sommer2013geometric}, which can be viewed as special cases of our elliptic parameterization.}
\label{ability}
\begin{center}
\begin{small}
\resizebox{0.88\textwidth}{!}{
\begin{tabular}{lcccccccc}
\toprule
\multirow{3}{*}{Model}   & \multirow{3}{*}{\begin{tabular}{@{}c@{}}Geometry\\of Orth.\end{tabular}} &\multirow{3}{*}{\begin{tabular}{@{}c@{}}Dimension\\of Orth.\end{tabular}} & \multicolumn{4}{c}{Logical Patterns} & \multicolumn{2}{c}{Topological Structures}                                           \\ \cmidrule(r){4-7} \cmidrule(r){8-9}
\multicolumn{1}{c}{}  & \multicolumn{1}{c}{} & \multicolumn{1}{c}{} &  Symmetry & Antisymmetry & Inversion & Composition                              & Cyclicity  & Hierarchy\\
\midrule
RotatE~\cite{sun2019rotate}  & $\mathbb{E}$ & $2$ & \CheckmarkBold & \CheckmarkBold & \CheckmarkBold & \CheckmarkBold & \CheckmarkBold  & \XSolidBrush \\
Rotate3D~\cite{rotate3d}  & $\mathbb{E}$ & $3$ & \CheckmarkBold & \CheckmarkBold & \CheckmarkBold & \CheckmarkBold & \CheckmarkBold  & \XSolidBrush \\
DualE~\cite{duale}  & $\mathbb{E}$ & $3$ & \CheckmarkBold & \CheckmarkBold & \CheckmarkBold & \XSolidBrush  & \CheckmarkBold & \XSolidBrush \\
QuatE~\cite{quate}  & $\mathbb{E}$ & $4$ & \CheckmarkBold & \CheckmarkBold & \CheckmarkBold & \XSolidBrush  & \CheckmarkBold & \XSolidBrush \\
HousE~\cite{house}  & $\mathbb{E}$ & $k$ & \CheckmarkBold & \CheckmarkBold & \CheckmarkBold & \CheckmarkBold  & \CheckmarkBold  & \XSolidBrush \\
RefH~\cite{RotH}  & $\mathbb{Q}$ & $2$  & \CheckmarkBold & \CheckmarkBold & \CheckmarkBold & \CheckmarkBold  & \XSolidBrush & \CheckmarkBold \\
RotH~\cite{RotH}  & $\mathbb{Q}$ & $2$  & \CheckmarkBold & \CheckmarkBold & \CheckmarkBold & \CheckmarkBold  & \XSolidBrush & \CheckmarkBold \\
AttH~\cite{RotH}  & $\mathbb{Q}$ & $2$  & \CheckmarkBold & \CheckmarkBold & \CheckmarkBold & \CheckmarkBold  & \XSolidBrush & \CheckmarkBold \\
\midrule
GoldE  & $\mathbb{P}(\mathbb{E}), \mathbb{Q}$ & $k$  & \CheckmarkBold & \CheckmarkBold & \CheckmarkBold & \CheckmarkBold  & \CheckmarkBold & \CheckmarkBold \\
\bottomrule
\end{tabular}}
\end{small}
\end{center}
\vspace{-4mm}
\end{table*}

However, 
\emph{it takes two to tango}---none of existing approaches can simultaneously break the restrictions of dimension and geometry for orthogonal relation transformations.
As shown in Table~\ref{ability}, 
the dimensional extensions are entrenched in Eu-clidean geometry,
while the hyperbolic relational orthogo-nalization remains constrained in low-dimensional spaces due to the significant computational demand and complexity.
Moreover, all these approaches are specifically designed on homogeneous geometry, which is inadequate to preserve the topological heterogeneity of KGs---in some regions cyclical, in others hierarchical \cite{mix1, mix2, mix3}.
In light of such limitations, a challenging question arises: 
\emph{can we generalize these approaches in both dimension and geometry to achieve the best of all worlds?}

In this paper, we give an affirmative answer by
presenting a general framework named GoldE, which represents relations with a \emph{universal orthogonal parameterization}.
To establish the universality, 
we first derive a generalized form of 
Householder reflection \cite{householder1958unitary} based on a quadratic inner product.
By treating such reflection as the elementary operator,
we can then theoretically design elegant mappings to parameterize orthogonal relation transformations of typical geometric types (i.e., Euclidean, elliptic and hyperbolic).
Each of them breaks through the dimensional rigidity without losing degrees of freedom, 
enabling superior capacity for modeling logical patterns and corresponding topologies.
Considering the inherent topological heterogeneity of KGs, we further integrate the designed mappings within a product manifold \cite{product1} to unify multiple
geometric types of orthogonal transformations.
In this way, such parameterization can simultaneously achieve dimensional extension and geometric unification for relational orthogonalization,
empowering our framework to better match heterogeneous topologies and thus provide higher-quality representations.

\textbf{Contributions:} To the best of our knowledge, GoldE is the first framework that generalizes existing KGE approaches \emph{in both dimension and geometry} of the orthogonal relation transformations.
With theoretical guarantees, we establish a \emph{universal orthogonal parameterization}, based on which our GoldE can \emph{simultaneously capture crucial logical patterns and inherent topological heterogeneity} as shown in Table~\ref{ability}.
Empirically, we conduct extensive experiments over three standard benchmarks, and GoldE consistently outperforms current state-of-the-art baselines across all the datasets.

\vspace{-1.5mm}
\section{Preliminaries}

\subsection{Problem Setup}
Given the entity set $\mathcal{V}$ and relation set $\mathcal{R}$, a knowledge graph can be formally defined as a collection of factual triples $\mathcal{F}=\{(h,r,t)\}\subseteq\mathcal{V}\times\mathcal{R}\times\mathcal{V}$, where each triple represents that there is a relation $r$ between head entity $h$ and tail entity $t$.
As an effective technique for automatically inferring missing links, KGE encodes entities and relations into expressive representations, and measures the plausibility of each triple with a pre-defined score function.

\begin{figure*}[t!]
\centering
\,\,\,
\subfigure[]{\label{Fig-2a}\includegraphics[width=0.45\columnwidth]{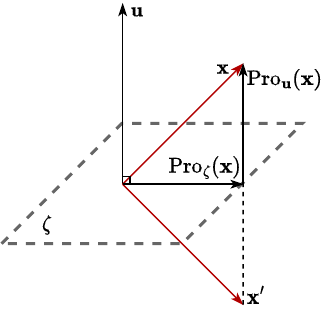}}
\quad
\subfigure[]{\label{Fig-2b}\includegraphics[width=0.45\columnwidth]{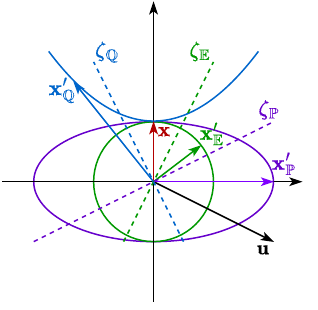}}
\quad
\subfigure[]{\label{Fig-2c}\includegraphics[width=0.45\columnwidth]{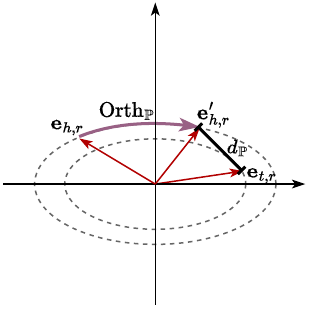}}
\quad
\subfigure[]{\label{Fig-2d}\includegraphics[width=0.45\columnwidth]{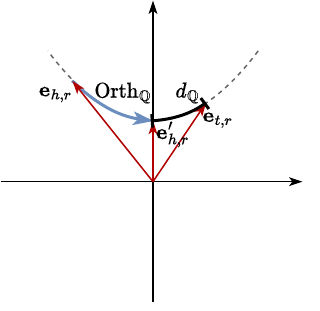}}
\vspace{-4mm}
\caption{(a) Illustration of generalized Householder reflection; (b) Three types of reflections of $\mathbf{x}$ about the hyperplanes $\zeta$ orthogonal to $\mathbf{u}$ under different weighting vectors, i.e., Euclidean (green lines), elliptic (purple lines) and hyperbolic (blue lines). Note that all quantities are displayed in the uniformly weighted Euclidean space, thus the non-Euclidean hyperplanes $\zeta_{\mathbb{P}}$ and $\zeta_{\mathbb{Q}}$ do not appear perpendicular to $\mathbf{u}$; (c) Elliptic orthogonal parameterization in $2$-dimensional space; (d) Hyperbolic orthogonal parameterization in $2$-dimensional space.}
\label{Fig-2}
\vskip -0.14in
\end{figure*}

\subsection{Geometric Background}
We briefly introduce some necessary background on elliptic and hyperbolic geometry. More relevant details are available in the standard texts~\cite{do1992riemannian, eisenhart1997riemannian}.

\textbf{Elliptic Manifold:}
Let $\mathbf{p}\in\mathbb{R}_{+}^{k}$ be a weighting vector where all elements are positive. For any $\mathbf{x},\mathbf{y}\in\mathbb{R}^{k}$, the elliptic inner product~\cite{e_def} is defined as: 
\begin{equation}
\begin{split}
\label{ell_inner}
\braket{\mathbf{x}, \mathbf{y}}_{\mathbf{p}}
=\mathbf{x}^\top\mathrm{diag}(\mathbf{p})\mathbf{y}
=p_1x_1y_1+\cdots+p_kx_ky_k. \nonumber
\end{split}
\end{equation}
For any $\beta>0$, the elliptic manifold (ellipsoid) $\mathbb{P}_{\mathbf{p},\beta}^{k}$ is the  following submanifold in $\mathbb{R}^{k}$:
\begin{equation}
    \mathbb{P}_{\mathbf{p},\beta}^{k}=\{\mathbf{x}\in\mathbb{R}^{k}:\braket{\mathbf{x}, \mathbf{x}}_{\mathbf{p}}=\beta\}.
\vspace{-1.0mm}
\end{equation}
If one set ${\mathbf{p}}$ to an all-ones vector $\mathbf{1}$, $\braket{\mathbf{x}, \mathbf{y}}_{\mathbf{p}}=\sum_{i=1}^kx_iy_i$ is the Euclidean inner product, based on which the ellipsoid $\mathbb{P}_{{\mathbf{p}},\beta}^{k}$ is specialized into the hypersphere $\mathbb{S}_\beta^{k}$.

\textbf{Hyperbolic Manifold:}
This paper chooses the hyperboloid model of hyperbolic space for its simplicity and numerical stability~\cite{hyperboloid}.
Let $\mathbf{q}\in\mathbb{R}^{k}$ be a weighting vector with the first element equal to $-1$ and the remaining elements equal to $+1$.
For any $\mathbf{x},\mathbf{y}\in\mathbb{R}^{k}$, the hyperbolic (Lorentz) inner product is defined as:
\begin{equation}
\begin{split}
\label{hyp_inner}
\braket{\mathbf{x},\mathbf{y}}_{\mathbf{q}}=\mathbf{x}^\top\mathrm{diag}(\mathbf{q})\mathbf{y}
=-x_1y_1+x_2y_2+\cdots+x_ky_k. \nonumber
\end{split}
\end{equation}
For any $\beta>0$, the hyperboloid $\mathbb{Q}_{\beta}^{k}$ is denoted as:
\begin{equation}
\label{hyp_define}
    \mathbb{Q}_{\beta}^{k}=\{\mathbf{x}\in\mathbb{R}^{k}:\braket{\mathbf{x}, \mathbf{x}}_{\mathbf{q}}=-\beta, x_1>0\}.
\vspace{-1.5mm}
\end{equation}

For the sake of clarification, all notations used in this paper are listed in Appendix~\ref{notations}.

\section{Methodology}
\subsection{Generalized Orthogonal Matrix and Mapping}
\label{house}
In the first step, we seek an elegant parameterization method for modeling orthogonal matrices in Euclidean, elliptic and hyperbolic spaces of arbitrary dimension $k$.
To establish this universality, we take a broader perspective by introducing a generalized space equipped with a quadratic inner product.

\textbf{Quadratic form:} For any $\mathbf{x},\mathbf{y}\in\mathbb{R}^{k}$, the quadratic inner product $\braket{\mathbf{x},\mathbf{y}}_{\mathbf{w}}$ with respect to $\mathbf{w}\in\mathbb{R}^{k}$ is defined as:
\begin{equation}
\begin{split}
\braket{\mathbf{x},\mathbf{y}}_{\mathbf{w}}=\mathbf{x}^\top\mathrm{diag}(\mathbf{w})\mathbf{y}=\sum_{i=1}^kw_ix_iy_i, \\
\end{split}
\end{equation}
where $\mathbf{w}$ is a weighting vector with no elements equal to $0$, ensuring that $\braket{\mathbf{x}, \mathbf{y}}_\mathbf{w}$ is non-degenerate.

\textbf{Generalized orthogonal matrix:} In the generalized space endowed with $\braket{\mathbf{x}, \mathbf{y}}_\mathbf{w}$, a real square matrix $\mathbf{G}\in\mathbb{R}^{k\times k}$ is orthogonal if it satisfies:
\begin{equation}
\label{inner-inv}
    \braket{\mathbf{Gx}, \mathbf{Gy}}_\mathbf{w}=\braket{\mathbf{x},\mathbf{y}}_\mathbf{w}\Leftrightarrow \mathbf{G}^\top\mathrm{diag}(\mathbf{w})\mathbf{G}=\mathrm{diag}(\mathbf{w}).
\end{equation}
We call the set of all $\mathbf{G}$ as the generalized orthogonal group in dimension $k$, denoted by $\mathbf{O}_{\mathbf{w}}(k)$. 
According to Equation (\ref{ell_inner}) and (\ref{hyp_inner}), the generalized orthogonal matrix $\mathbf{G}\in\mathbf{O}_{\mathbf{w}}(k)$ can be naturally specialized into Euclidean, elliptic and hyperbolic spaces by setting $\mathbf{w}$ to $\mathbf{1}$, $\mathbf{p}$ and $\mathbf{q}$, respectively.
Based on this observation, we turn to modeling such generalized orthogonal matrices for establishing the universality.

\textbf{Generalized Householder matrix:} Before fully covering all $\mathbf{G}\in \mathbf{O}_{\mathbf{w}}(k)$, we start with the simplest case: \emph{elementary reflection matrix}, which is also known as the Householder matrix~\cite{householder1958unitary} in Euclidean space.
Here we derive the generalized form of Householder matrix in the quadratic inner product space.

Geometrically, the generalized Householder matrix $\mathbf{H}$ describes an elementary reflection from $\mathbf{x}$ to $\mathbf{x}'$ about a a hyperplane $\zeta$ as shown in Figure~\ref{Fig-2a}.
Given a normal vector $\mathbf{u}\in\mathbb{R}^{k}$ of $\zeta$, the projection of $\mathbf{x}$ onto $\mathbf{u}$ is denoted as:
\begin{equation}
    \mathrm{Pro}_\mathbf{u}(\mathbf{x})=\frac{\braket{\mathbf{u},\mathbf{x}}_\mathbf{w}}{\braket{\mathbf{u},\mathbf{u}}_\mathbf{w}}\mathbf{u}.
\end{equation}
By the law of vector addition, the projection of $\mathbf{x}$ onto the hyperplane $\zeta$, i.e., $\mathrm{Pro}_\zeta(\mathbf{x})$, is calculated as:
\begin{equation}
    \mathrm{Pro}_\zeta(\mathbf{x})=\mathbf{x}-\mathrm{Pro}_\mathbf{\mathbf{u}}(\mathbf{\mathbf{x}}).
\end{equation}
Based on the concept of projections, we move to the elemen-tary reflection of $\mathbf{x}$ with respect to $\zeta$ in the quadratic inner product space. The resulting vector $\mathbf{x}'$ is given by:
\begin{equation}
\begin{split}
\label{ref-new}
    \mathbf{x}'&=\mathrm{Pro}_\zeta(\mathbf{x})-\mathrm{Pro}_\mathbf{u}(\mathbf{x})
    = \underbrace{(\mathbf{I}-2\frac{\mathbf{u}\mathbf{u}^\top\mathrm{diag}(\mathbf{w})}{\mathbf{u}^\top\mathrm{diag}(\mathbf{w})\mathbf{u}})}_{\mathbf{H}(\mathbf{u},\mathbf{w})}\mathbf{x},
\end{split}
\end{equation}
where $\mathbf{I}\in\mathbb{R}^{k\times k}$ is the identity matrix. 
Notably, the induced reflection matrix $\mathbf{H}(\mathbf{u},\mathbf{w})$, with normal vector $\mathbf{u}$ and weighting vector $\mathbf{w}$ as variables, can be viewed as a generalization of Euclidean Householder matrix ($\mathbf{w}=\mathbf{1}$) in the quadratic inner product space.
It is straightforward to demonstrate that the generalized Householder matrix $\mathbf{H}(\mathbf{u},\mathbf{w})$ is orthogonal, which satisfies the inner product invariance in Equation (\ref{inner-inv}).
Figure~\ref{Fig-2b} illustrates the Euclidean, elliptic and hyperbolic reflections in $2$-dimensional space, described by $\mathbf{H}$ with $\mathbf{w}$ equal to $\mathbf{1}$, $\mathbf{p}$ and $\mathbf{q}$, respectively.

\textbf{Orthogonal mapping:} By taking the elementary reflections as basic operators, we can design a mapping to represent the generalized orthogonal transformations.
Formally, given a series of vectors $\mathbf{U}=\{\mathbf{u}_c\}_{c=1}^n$ where $\mathbf{u}_c\in\mathbb{R}^{k}$, we define the mapping with respect to $\mathbf{w}\in\mathbb{R}^k$ as follows:
\begin{equation}
\label{orth}
    \mathrm{Orth}(\mathbf{U}, \mathbf{w})=\prod_{c=1}^n \mathbf{H}(\mathbf{u}_c,\mathbf{w}).
\end{equation}
Since each $\mathbf{H}(\mathbf{u}_c,\mathbf{w})$ preserves the quadratic inner product, one can easily verify that the output of $\mathrm{Orth}(\mathbf{U}, \mathbf{w})$ is also a generalized orthogonal matrix according to Equation (\ref{inner-inv}).
Moreover, we have the following theorem:
\begin{theorem}
When $n=k$, the image of $\mathrm{Orth}$ is the set of all $k\times k$ generalized orthogonal matrices, i.e., ${\rm Image( \mathrm{Orth})}=\mathbf{O}_{\mathbf{w}}(k)$. (See proof in Appendix \ref{proof-thm-1})
\label{thm-1}
\end{theorem}

This theorem guarantees that the output of $\mathrm{Orth}$ completely covers the generalized orthogonal group $\mathbf{O}_{\mathbf{w}}(k)$.
Therefore, such a mapping can naturally represent all the orthogonal matrices of Euclidean, elliptic and hyperbolic types with the weighting vector $\mathbf{w}$ equal to $\mathbf{1}$, $\mathbf{p}$ and $\mathbf{q}$, respectively.
\subsection{GoldE: Universal Orthogonal Parameterization}
\label{golde}

In the following, 
we leverage the derived mapping $\mathrm{Orth}$ as the cornerstone to develop the \emph{elliptic} (Section \ref{ell_p}) and \emph{hyperbolic} (Section \ref{hyp_p}) orthogonal parameterizations for modeling relations,
which are further integrated in a product manifold to establish our GoldE framework (Section \ref{mix-section}).

\subsubsection{Elliptic Orthogonal Parameterization}
\label{ell_p}
To better capture logical patterns and cyclical structures of KGs, we propose to model each relation as a $k$-dimensional elliptic orthogonal transformation between head and tail entities with the designed mapping $\mathrm{Orth}$ in Equation (\ref{orth}).

Given an elliptic weighting vector $\mathbf{p}\in\mathbb{R}_{+}^{k}$, the initial embedding $\mathbf{e}_v\in\mathbb{R}^k$ for entity $v\in\mathcal{V}$ can be regarded as lying on the ellipsoid $\mathbb{P}_{\mathbf{p},\beta_{\mathbf{e}_v,\mathbf{p}}}^k$, where $\beta_{\mathbf{e}_v,\mathbf{p}}=\braket{\mathbf{e}_v,\mathbf{e}_v}_{\mathbf{p}}$.
Considering that each entity typically exhibits different characteristics when involving different relations \cite{transh},
we define a learnable weighting vector $\mathbf{p}_r$ for each relation $r\in\mathcal{R}$ to associate entities with the $r$-specific ellipsoid.
For each triple $(h,r,t)$, the $r$-associated embeddings of head entity $h$ and tail entity $t$ are denoted as $\mathbf{e}_{h,r}\in\mathbb{P}_{\mathbf{p}_r,\beta_{\mathbf{e}_{h},\mathbf{p}_r}}^k$ and $\mathbf{e}_{t,r}\in\mathbb{P}_{\mathbf{p}_r,\beta_{\mathbf{e}_{t},\mathbf{p}_r}}^k$.
We then represent relation $r$ as an elliptic orthogonal transformation between $\mathbf{e}_{h,r}$ and $\mathbf{e}_{t,r}$.
Formally, the embedding of relation $r$ is defined as $\mathbf{U}_r\in\mathbb{R}^{k\times k}$, in which each row $\mathbf{U}_r[i]\in\mathbb{R}^{k}$ is a normal vector for elementary reflection.
Taking $\mathbf{U}_r$ and $\mathbf{p}_r$ as input to the mapping $\mathrm{Orth}$, we can apply the $r$-specific elliptic orthogonal transformation to the head embedding $\mathbf{e}_{h,r}$:
\begin{equation}
\begin{split}
\label{ell_orth}
    \mathbf{e}'_{h,r} &= \mathrm{Orth}_\mathbb{P}(\mathbf{U}_r, \mathbf{p}_r)\mathbf{e}_{h,r} \\
    &= \prod_{i=1}^k \mathbf{H}(\mathbf{U}_r[i],\mathbf{p}_r)\mathbf{e}_{h,r}.
\end{split}    
\end{equation}
Based on Theorem~\ref{thm-1}, any $k$-dimensional elliptic orthogonal relation transformation can be naturally represented by Equation (\ref{ell_orth}) without any loss of degrees of freedom.


\textbf{Objective function:} If triple $(h, r, t)$ holds, the transformed head embedding $\mathbf{e}'_{h,r}$ is expected to be close to the tail embedding $\mathbf{e}_{t,r}$.
Figure~\ref{Fig-2c} illustrates the elliptic orthogonal parameterization in $2$-dimensional space.
Note that $\mathbf{e}'_{h,r}$ and $\mathbf{e}_{t,r}$ lie on the ellipsoids that have the same $\mathbf{p}_r$ but different $\beta$.
To measure such proximity, we introduce the extrinsic Mahalanobis distance~\cite{mitchell1985mahalanobis} as the score function for each $(h,r,t)$:
\begin{align}
  s_r(h,t)&=-d_{\mathbb{P}}(\mathbf{e}'_{h,r},\mathbf{e}_{t,r})\label{ell_score}\\
  &=-\sqrt{(\mathbf{e}'_{h,r}-\mathbf{e}_{t,r})^\top\mathrm{diag}(\mathbf{p}_r)(\mathbf{e}'_{h,r}-\mathbf{e}_{t,r})}. \notag
\end{align}
Note that it is also possible to normalize all entities onto one ellipsoid and measure similarity using the elliptic geodesic distance.
However, the extrinsic measurement usually works better in practice \cite{sun2019rotate, mix3}.



\textbf{Connections to Euclidean models:}
As shown in Table~\ref{ability}, a series of existing models such as RotatE, QuatE and HousE represent relations as Euclidean orthogonal transformations with respective dimensions.
All these models can be viewed as the special cases of the proposed elliptic parameterization by setting $\mathbf{p}_r$ to $\mathbf{1}$ for all $r\in\mathcal{R}$.
With the designed orthogonal mapping $\mathrm{Orth}$, our parameterization is applicable to the space of any dimension $k$.
Moreover, we highlight that such elliptic generalization moves beyond Euclidean models by introducing an additional relation-specific scaling operation.
Formally, we can prove the following claim:
\begin{claim}
\label{pro1}
The proposed elliptic parameterization can be reformulated as Euclidean parameterization equipped with element-wise scaling transformations determined by $\sqrt{\mathbf{p}_r}$. (See proof in Appendix~\ref{proof-pro-1})
\end{claim}
Such scaling transformations enable an adaptive adjustment of the margin in loss function to fit for complex relations \cite{PairRE}. 
However, elliptic geometry is inherently not the optimal choice for preserving the hierarchical structures due to its mismatched geometric characteristics,
which then motivates us to investigate hyperbolic parameterization.

\subsubsection{Hyperbolic Orthogonal Parameterization}
\label{hyp_p}
As a complement for preserving the hierarchical structures of KGs with low distortion, we further represent relations as $k$-dimensional hyperbolic orthogonal transformations.

In order to rigorously measure similarity with the hyperbolic geodesic distance\footnote[1]{Another alternative is the extrinsic Lorentzian distance~\cite{lorentz_distance}. However, such distance does not satisfy the triangle inequality and may underestimate the plausibility of triples.},
we begin by projecting the initial entity embeddings $\mathbf{e}_h$ and $\mathbf{e}_t$ to the hyperboloid via the hyperbolic exponential map $g_\beta:\mathbb{R}^{k-1}\rightarrow\mathbb{Q}_\beta^k$.
Given any $\mathbf{x}\in\mathbb{R}^{k-1}$, $g_\beta$ is defined as~\cite{stable}:
\begin{equation}
\label{hyp_exp}
    g_\beta(\mathbf{x})=(\sqrt{\beta}\:\mathrm{cosh}(\frac{\Vert\mathbf{x}\Vert}{\sqrt{\beta}}),\sqrt{\beta}\:\mathrm{sinh}(\frac{\Vert\mathbf{x}\Vert}{\sqrt{\beta}})\frac{\mathbf{x}}{\Vert \mathbf{x}\Vert}).
\end{equation}
Following~\cite{RotH}, we define a learnable radius of curvature $\beta_r$ for each relation $r$ to encode a variety of hierarchies.
For each triple $(h,r,t)$, $\mathbf{e}_h$ and $\mathbf{e}_t$ are mapped to the $r$-specific hyperboloid $\mathbb{Q}_{\beta_r}^k$ via $g_{\beta_r}$, resulting in the $r$-associated embeddings $\mathbf{e}_{h,r} = g_{\beta_r}(\mathbf{e}_h)$ and $\mathbf{e}_{t,r} = g_{\beta_r}(\mathbf{e}_t)$.

For parameterizing hyperbolic orthogonal transformation of any dimension $k$ between $\mathbf{e}_{h,r}$ and $\mathbf{e}_{t,r}$, a straightforward strategy is to leverage the designed orthogonal mapping $\mathrm{Orth}$ in the case of $\mathbf{w}:=\mathbf{q}=(-1, 1, \ldots, 1)^\top$.
Given the relation embedding $\mathbf{U}_r\in\mathbb{R}^{k\times k}$,
the mapping $\mathrm{Orth}(\mathbf{U}_r,\mathbf{q})$ is able to represent every hyperbolic orthogonal relation matrix $\mathbf{G}_r\in\mathbf{O}_{\mathbf{q}}(k)$ according to Theorem~\ref{thm-1}.

Nevertheless, although $\mathbf{G}_r$ satisfies the inner product invariance in Equation (\ref{inner-inv}), it is not guaranteed to preserve the sign of the first element of $\mathbf{e}_{h,r}$. In other words, $\mathbf{e}'_{h,r}$ might be transformed to the lower sheet of the hyperboloid with the first element less than $0$, leading to the inability to calculate the geodesic distance between $\mathbf{e}'_{h,r}\notin\mathbb{Q}_{\beta_r}^k$ and $\mathbf{e}_{t,r}\in\mathbb{Q}_{\beta_r}^k$.
To remedy this deficiency, we restrict the orthogonal relation matrix $\mathbf{G}_r$ into the \emph{positive} subgroup of $\mathbf{O}_{\mathbf{q}}(k)$ such that for any $\mathbf{x}\in\mathbb{Q}_{\beta_r}^k$, the first element of $\mathbf{G}_r\mathbf{x}$ is positive, thereby ensuring that $\mathbf{G}_r\mathbf{x}\in\mathbb{Q}_{\beta_r}^k$.
Essentially, the full hyperbolic orthogonal group $\mathbf{O}_{\mathbf{q}}(k)$ can be classified into two distinct connected components based on the following proposition:
\begin{proposition}
\label{pro2}
    For any $\mathbf{x}\in\mathbb{Q}_{\beta}^k$, every hyperbolic orthogonal matrix $\mathbf{G}\in\mathbf{O}_{\mathbf{q}}(k)$ falls into two subsets according to the sign of its first element $G_{11}$ for which $|G_{11}|\geq1$:
    \begin{itemize}[leftmargin=*,labelsep=3.5pt]
    \vspace{-3mm}
        \item The matrix in the positive subset $\mathbf{O}^+_{\mathbf{q}}(k)=\{\mathbf{G}\in\mathbf{O}_{\mathbf{q}}(k): G_{11}\geq+1\}$ preserves the sign of the first element of $\mathbf{x}$;
    \vspace{-1mm}
        \item The matrix in the negative subset $\mathbf{O}^-_{\mathbf{q}}(k)=\{\mathbf{G}\in\mathbf{O}_{\mathbf{q}}(k): G_{11}\leq-1\}$ reverses the sign of the first element of $\mathbf{x}$;
    \vspace{-1mm}
        \item $\mathbf{O}^+_{\mathbf{q}}(k)$ is a multiplicative subgroup of \;\!$\mathbf{O}_{\mathbf{q}}(k)$.
    \vspace{-3mm}
    \end{itemize}
    (See proof in Appendix~\ref{proof-pro-2})
\end{proposition}
For fully covering the positive subgroup $\mathbf{O}^+_{\mathbf{q}}(k)$, we employ the polar decomposition~\cite{gallier2005notes} to express any $\mathbf{G}\in\mathbf{O}^+_{\mathbf{q}}(k)$ 
as the product of a Euclidean orthogonal matrix and a hyperbolic boost matrix.
Formally, we have the following proposition, based on which a positive hyperbolic mapping $\mathrm{Orth}_{\mathbb{Q}}$ is defined for relational parameterization:
\begin{proposition}
\label{pos-map}
    Every (positive) hyperbolic orthogonal matrix $\mathbf{G}\in\mathbf{O}^+_{\mathbf{q}}(k)$ can be expressed by a mapping $\mathrm{Orth}_{\mathbb{Q}}$ such that for $\mathbf{U}\in\mathbb{R}^{(k-1)\times(k-1)}$ and $\mathbf{b}\in\mathbb{R}^{k-1}$:
    \begin{align}
        G&=\mathrm{Orth}_{\mathbb{Q}}(\mathbf{U},\mathbf{b})
        \notag\\&={\renewcommand{\arraycolsep}{3pt}\begin{bmatrix}
  1 & 0\\ 
  0 & \mathrm{Orth}(\mathbf{U}, \mathbf{1})
\end{bmatrix}}{\renewcommand{\arraycolsep}{0pt}\begin{bmatrix}
  \sqrt{\Vert \mathbf{b} \Vert^2+1} & \mathbf{b}^\top\\ 
  \mathbf{b} & \sqrt{\mathbf{I}+\mathbf{b}\mathbf{b}^\top}
\end{bmatrix}}. \notag
    \end{align}
(See proof in Appendix~\ref{proof-pro-3})
\end{proposition}
Based on Proposition~\ref{pos-map}, we can represent each relation $r$ as a hyperbolic orthogonal transformation whose output space is restricted to $\mathbb{Q}_{\beta_r}^k$.
Specifically, two types of parameters are defined for each $r$: $\mathbf{U}_r\in\mathbb{R}^{(k-1)\times(k-1)}$ and $\mathbf{b}_r\in\mathbb{R}^{k-1}$.
Taking $\mathbf{U}_r$ and $\mathbf{b}_r$ as input to the mapping $\mathrm{Orth}_{\mathbb{Q}}$,
we apply such restricted orthogonal transformation to $\mathbf{e}_{h,r}\in\mathbb{Q}_{\beta_r}^k$:
\begin{align}
    &\mathbf{e}'_{h,r} = \mathrm{Orth}_{\mathbb{Q}}(\mathbf{U}_r,\mathbf{b}_r)\mathbf{e}_{h,r}\label{new_hyp_orth}\\
    &={\renewcommand{\arraycolsep}{3pt}\footnotesize\begin{bmatrix}
  1 & 0\\ 
  0 & \prod_{i=1}^{k-1} \mathbf{H}(\mathbf{U}_r[i],\mathbf{1})
\end{bmatrix}}{\renewcommand{\arraycolsep}{0pt}\footnotesize\begin{bmatrix}
  \sqrt{\Vert \mathbf{b}_r \Vert^2+1} & \mathbf{b}_r^\top\\ 
  \mathbf{b}_r & \sqrt{\mathbf{I}+\mathbf{b}_r\mathbf{b}_r^\top}
\end{bmatrix}}\mathbf{e}_{h,r}. \notag
\end{align}
\vspace{-3mm}

\textbf{Objective function:}
Proposition~\ref{pos-map} theoretically ensures that the transformed head embedding $\mathbf{e}'_{h,r}$ lies on the same hyperboloid as the tail embedding $\mathbf{e}_{t,r}$, enabling us to measure their similarity using the hyperbolic geodesic distance.
Formally, the score function for each $(h,r,t)$ is defined as:
\begin{equation}
\begin{split}
    s_r(h,t)&=-d_\mathbb{Q}(\mathbf{e}'_{h,r},\mathbf{e}_{t,r})\\
    &=-\sqrt{\beta}\:\mathrm{cosh^{-1}}(-\frac{\braket{\mathbf{e}'_{h,r},\mathbf{e}_{t,r}}_{\mathbf{q}}}{\beta}).
\end{split}
\end{equation}
\vspace{-4mm}


\textbf{Connections to hyperbolic models:} Existing hyperbolic models such as AttH, RotH and RefH \cite{RotH} represent each relation as the combination of a rotation (or reflection) and a M\"{o}bius addition in the Poincar\'e ball.
Since the Poincar\'e ball is isometric to the hyperboloid \cite{hyp1}, and the M\"{o}bius addition is equivalent to the hyperbolic boost \cite{hyp_hierachy}, all these models can be viewed as the special cases of our proposed hyperbolic parameterization.
Note that these models param-eterize rotations/reflections using the $2$-dimensional Givens transformations, while our approach breaks the dimensional restriction without loss of degrees of freedom, thereby generalizing them and achieving superior modeling capacity.

\subsubsection{Mixed Orthogonal Parameterization}
\label{mix-section}
Although our elliptic and hyperbolic parameterizations are advantageous for embedding the corresponding topologies (cyclicity or hierarchy), real-world KGs typically exhibit non-uniform structures~\cite{mix1, mix3}. 
Considering such inherent topological heterogeneity, we further introduce a mixed orthogonal parameterization in the product manifold~\cite{mix1} to unify multiple geometric spaces, culminating in the universal GoldE framework.

\textbf{Product manifold:}
Beyond individual elliptic or hyperbolic spaces, we leverage the Cartesian product to describe a heterogeneous embedding space for GoldE: $\mathbb{D}^k=\bigtimes_{i=1}^m\mathbb{X}_{i}^{k_i}$,
where $\mathbb{X}_{i}^{k_i}\in\{\mathbb{E},\mathbb{P},\mathbb{Q}\}$\footnote[2]{Note that the notation $\mathbb{E}$ is redundant since Euclidean orthogonal parameterization is the special case of our proposed elliptic orthogonal parameterization according to Claim~\ref{pro1}.} denotes the $i$-th component space of dimension $k_i$\footnote[3]{One can assign identical dimensions to component spaces of the same geometric type (i.e., $k_i\in\{k_\mathbb{P},k_\mathbb{Q}\}$) for reducing the number of hyperparameters~\cite{mix1}.} such that the total dimension $k=\sum_{i=1}^m k_i$.
The resulting product space is a “non-constantly” curved manifold
that is more flexible than a single constant curvature manifold~\cite{mix2}.


For each triple $(h,r,t)$, GoldE represents relation $r$ as a component-wise orthogonal transformation between head entity $h$ and tail entity $t$ in the product manifold.
Formally, let $\mathbf{e}_{h,r}^{(i)}$ and $\mathbf{e}_{t,r}^{(i)}$ denote the $i$-th sub-embeddings of $h$ and $t$, which are associated to the $r$-specific (elliptic or hyperbolic) space as described in Section \ref{ell_p} and \ref{hyp_p}.
GoldE applies the generalized 
orthogonal transformation to $\mathbf{e}_{h,r}^{(i)}$ according to the geometric type of the $i$-th component space $\mathbb{X}_i$:
\begin{align}
\label{mix-orth}
{\mathbf{e}'}_{\hspace{-0.25em}h,r}^{(i)} = 
  \begin{cases}
    \mathrm{Orth}_\mathbb{P}(\mathbf{U}_r^{(i)}, \mathbf{p}_r^{(i)})\mathbf{e}_{h,r}^{(i)}, & \text{for } \mathbb{X}_i=\mathbb{P} \\[1ex]
    \mathrm{Orth}_\mathbb{Q}(\mathbf{U}_r^{(i)},\mathbf{b}_r^{(i)})\mathbf{e}_{h,r}^{(i)}, & \text{for } \mathbb{X}_i=\mathbb{Q}
  \end{cases}
  ,
\end{align}
where $\mathbf{U}_r^{(i)}$, $\mathbf{p}_r^{(i)}$ and $\mathbf{b}_r^{(i)}$ are the parameters of relation transformations in the $i$-th component space.

\textbf{Objective function:}
Since the orthogonal transformations guarantee the inner product invariance in each component space, the transformed head embedding $\mathbf{e}'_{h,r}$ remains in the product space $\mathbb{D}^k$.
The distance between $\mathbf{e}'_{h,r}$ and $\mathbf{e}_{t,r}$ can be calculated by performing decomposition~\cite{product1, product2}. Formally, the score function for measuring the plausibility of each $(h,r,t)$ is defined as:
\begin{equation}
\begin{split}
    s_r(h,t)&=-d_\mathbb{D}(\mathbf{e}'_{h,r},\mathbf{e}_{t,r}) 
    =-\sum_{i=1}^m d_{\mathbb{X}_i}^l({\mathbf{e}'}_{\hspace{-0.25em}h,r}^{(i)},\mathbf{e}_{t,r}^{(i)}),
\end{split}
\end{equation}
where $l$ is the norm indicator.
We highlight that such mixed orthogonal parameterization in the product space naturally achieves the geometric unification, enabling improved representations by better matching the geometry of the embedding space to the heterogeneous structures of KGs.

\textbf{Complexity analysis:} 
The learnable parameters of GoldE include $\{\mathbf{e}_v\}_{v\in\mathcal{V}}$ and
$\{\mathbf{U}_r,\mathbf{p}_r,\mathbf{b}_r,\beta_r\}_{r\in\mathcal{R}}$.
Compared to the previous models~\cite{sun2019rotate,RotH,house},
the extra space cost is proportional to the number of relation types, which is usually much smaller than the number of entities.
Therefore, GoldE has the same space complexity as the existing KGE models, i.e., $O(k|\mathcal{V}|)$.
In the aspect of time cost for processing a single triple,
one may observe that the time complexity of Equation (\ref{mix-orth}) is $O(k^3_i)$, in which $k_i$ matrix-vector multiplications incur high computational costs.
However, it is worth noting that these multiplications can be entirely replaced by vector operations via efficient computation (refer to Appendix~\ref{efficient} for details), thereby reducing the time complexity to $O(k^2_i)$.


\textbf{Modeling capability:}
Benefiting from the dimensional extension and geometric unification achieved by the designed parameterization, GoldE is able to simultaneously model the logical patterns and heterogeneous structures as shown in Table~\ref{ability}.
Formally, we have the following claim:
\begin{claim}
\label{pro-ab}
    GoldE is capable of capturing the symmetry, antisymmetry, inversion and composition patterns, as well as the inherent cyclical and hierarchical structures in KGs. (See proof in Appendix~\ref{proof-pro-4})
\end{claim}


\subsection{Optimization}

Following \cite{sun2019rotate, rotate3d, house}, we use the self-adversarial negative sampling loss \cite{sun2019rotate} to optimize our GoldE framework.
Formally, given a positive triple $(h,r,t)$, the loss function is defined as:
\begin{equation}
\begin{split}
    L = &- \log{\sigma(\gamma+s_r(h,t))}\\
    &- \sum_{i=1}^g p(h'_i,r,t'_i)\log{\sigma(-s_r(h'_i,t'_i)-\gamma)},
\end{split}
\end{equation}
where $\gamma$ is a pre-defined margin, $\sigma$ is the sigmoid function, $g$ is the number of negative samples,  $(h'_i,r,t'_i)$ is the $i$-th negative sample against $(h,r,t)$, and 
$p(h'_i,r,t'_i)$ determines the proportion of $(h'_i,r,t'_i)$ in the current optimization as defined in \cite{sun2019rotate}.
Note that we initially define the entity embeddings of GoldE in Euclidean space, and all the relation transformations (including $g_\beta$) are differentiable and bijective \cite{HGCN}.
This enables us to learn the parameters using the standard Euclidean optimization techniques for the numerical stability \cite{stable}.
\vspace{-0.1mm}

\section{Experiments}
\begin{table*}[t]
\caption{Link prediction results on WN18RR, FB15k-237 and YAGO3-10. Best results are in \textbf{bold} and second best results are \underline{underlined}. $[\dagger]$: Results are taken from~\cite{ConvKB}; $[\diamond]$: Results are taken from~\cite{wn18rr}. $[\ast]$: HousE-pro is a variant of HousE with additional projection operations \cite{house}. Other results are taken from the corresponding original papers.}
\label{WN18RR-and-FB15k237}
\begin{center}
\begin{small}
\resizebox{0.9\textwidth}{!}{
\begin{tabular}{lccccccccccccccc}
\toprule
 \multirow{3}{*}{Model}    & \multicolumn{5}{c}{WN18RR}                                                                                                                                                                           & \multicolumn{5}{c}{FB15k-237}               & \multicolumn{5}{c}{YAGO3-10}  \\ \cmidrule(r){2-6} \cmidrule(r){7-11} \cmidrule(r){12-16}
\multicolumn{1}{c}{}    & MR                                   & MRR                                   & H@1                                   & H@3                                   & H@10                                  & MR                                  & MRR                                   & H@1                                   & H@3                                   & H@10                                  & MR                   & MRR                  & H@1                  & H@3                  & H@10                 \\
\midrule
TransE\;\!$\dagger$                  & 3384                                 & .226                                 & -                                     & -                                     & .501                                 & 357                                 & .294                                 & -                                     & -                                     & .465                                  & -                    & -                    & -                    & -                    & -                    \\
DistMult\;\!$\diamond$             & 5110                                 & .430                                  & .390                                  & .440                                  & .490                                  & 254                                 & .241                                 & .155                                 & .263                                 & .419                                 & 5926                 & .340                  & .240                  & .380                  & .540                  \\
ComplEx\;\!$\diamond$               & 5261                                 & .440                                  & .410                                  & .460                                  & .510                                  & 339                                 & .247                                 & .158                                 & .275                                 & .428                                 & 6351                 & .360                  & .260                  & .400                   & .550                  \\
ConvE\;\!$\diamond$                   & 4187                                 & .430                                  & .400                                  & .440                                  & .520                                  & 224                                 & .325                                 & .237                                 & .356                                 & .501                                 & 1671                 & .440                  & .350                  & .490                  & .620                  \\
MuRP                & -                                 & .481                                 & .440                                 & .495                                 & .566                                 & -                                 & .335                                 & .243                                 & .367                                 & .518                                 & -                    & .354                    & .249                    & .400                    & .567                    \\
\midrule
RotatE                 & 3340                                 & .476                                 & .428                                 & .492                                 & .571                                 & 177                                 & .338                                 & .241                                 & .375                                 & .533                                 & 1767                 & .495                 & .402                 & .550                  & .670                  \\
Rotate3D                 & 3328                                 & .489                                 & .442                                 & .505                                 & .579                                 & 165                                 & .347                                 & .250                                 & .385                                 & .543                                 & -                 & -                 & -                 & -                  & -                  \\
QuatE                & 3472                                 & .481                                 & .436                                 & .500                                 & .564                                 & 176                                 & .311                                 & .221                                 & .342                                 & .495                                 & -                    & -                    & -                    & -                    & -                    \\
DualE                  & -                                    & .482                                 & .440                                 & .500                                 & .561                                 & -                                   & .330                                 & .237                                 & .363                                 & .518                                 & -                    & -                    & -                    & -                    & -                    \\
ReflectE                & 1730                                    & .488                                 & .450                                 & .501                                 & .559                                 & \underline{148}                                   & .358                                 & .263                                 & .396                                 & .546                         & -                    & -                    & -                    & -                    & -                    \\
HousE  & 1885 & .496 & .452 & .511 & .585 & 165 & .348 & .254 & .384 & .534 & 1449 & .565 & .487 & .616 & .703 \\
HousE-pro\;\!$\ast$  & \underline{1303} & \underline{.511} & \underline{.465} & \underline{.528} & \underline{.602} & 153 & \underline{.361} & \underline{.266} & \underline{.399} & \underline{.551} & \underline{1415} & .571 & .491 & \underline{.620} & \underline{.714} \\
CompoundE   & - & .491 & .450 & .508 & .576 & - & .357 & .264 & .393 & .545 & - & - & - & - & - \\
\midrule
RefH                  & -                                 & .461                                 & .404                                 & .485                                 & .568                                 & -                                 & .346                                 & .252                                 & .383                                 & .536                                 & -                    & \underline{.576}                    & \underline{.502}                    & .619                    & .711                    \\
RotH                 & -                                 & .496                                 & .449                                 & .514                                 & .586                                 & -                                 & .344                                 & .246                                 & .380                                 & .535                                 & -                    & .570                    & .495                    & .612                    & .706                    \\
AttH                 & -                                 & .486                                 & .443                                 & .499                                 & .573                                 & -                                 & .348                                 & .252                                 & .384                                 & .540                                 & -                    & .568                    & .493                    & .612                    & .702                    \\
\midrule
GoldE  &\textbf{790} & \textbf{.525} & \textbf{.476} & \textbf{.542} & \textbf{.615} & \textbf{142} & \textbf{.370} & \textbf{.277} &  \textbf{.405} & \textbf{.558} & \textbf{536} & \textbf{.588} & \textbf{.510} & \textbf{.634} & \textbf{.723}\\
\bottomrule
\end{tabular}}
\end{small}
\end{center}
\vspace{-8mm}
\end{table*}

To comprehensively validate the effectiveness of GoldE, we conduct extensive experiments on link prediction task, and attempt to answer the following questions:

(1) How does GoldE perform on link prediction compared to existing KGE approaches? (Section~\ref{main})

(2) What is the effect of configuring different geometries and dimensions for orthogonal relation transformations on the performance of GoldE? (Section~\ref{analysis})

(3) From a fine-grained perspective, does GoldE effectively model different relation types? (Section~\ref{fine-grained})

(4) Can GoldE still learn high-quality representations even when the embedding size $k$ is restricted? (Section~\ref{restricted_k})

\subsection{Experimental Setup}
\textbf{Datasets:} We evaluate GoldE framework on three  standard benchmarks: WN18RR \cite{wn18rr}, FB15k-237 \cite{fb15k237} and YAGO3-10 \cite{yago3-10}.
Refer to Appendix~\ref{Datasets} for statistical details.


\textbf{Baselines:}
We compare GoldE to a series of advanced KGE approaches, including (1) non-orthogonal models: TransE \cite{TransE}, DistMult \cite{distmult}, ComplEx \cite{complex},  
ConvE \cite{wn18rr} and MuRP \cite{MuRP}; 
(2) Euclidean orthogonal models: RotatE \cite{sun2019rotate}, Rotate3D \cite{rotate3d}, QuatE \cite{quate}, DualE \cite{duale}, ReflectE \cite{ReflectE}, HousE \cite{house} and CompoundE \cite{CompoundE}; 
(3) hyperbolic orthogonal models: RefH, RotH and AttH \cite{RotH}.

\textbf{Implementation details:}
For fair comparisons, we follow \cite{house} to fix the embedding dimension $k$, ensuring that the number of parameters is comparable to the baselines.
Moreover, we follow the most natural way \cite{mix1} to construct the product manifold $\mathbb{D}$ as the combination of both $\mathbb{P}$ and $\mathbb{Q}$ for GoldE, i.e., $\mathbb{D}^k=\bigtimes_{i=1}^{m_\mathbb{P}}\mathbb{P}_{i}^{k_\mathbb{P}}\times \bigtimes_{j=1}^{m_\mathbb{Q}}\mathbb{Q}_{j}^{k_\mathbb{Q}}$.
The hyperparameters $m_\mathbb{P}$ and $k_\mathbb{P}$ ($m_\mathbb{Q}$ and $k_\mathbb{Q}$) determine the number and dimension of elliptic (hyperbolic) components, respectively.
More details can be found in Appendix~\ref{imp}.

\begin{table}[t]
\caption{Link prediction results for our GoldE framework under different geometric combinations.}
\label{geometry-ablation}
\vspace{-1.0mm}
\begin{center}
\begin{small}
\resizebox{0.98\columnwidth}{!}{
\begin{tabular}{ccccccc}
\toprule
  \multirow{3}{*}{\begin{tabular}{@{}c@{}}Geometry of Orth.\end{tabular}}       & \multicolumn{2}{c}{WN18RR} & \multicolumn{2}{c}{FB15k-237} & \multicolumn{2}{c}{YAGO3-10} \\ \cmidrule(r){2-3} \cmidrule(r){4-5} \cmidrule(r){6-7}
\multicolumn{1}{c}{} & MRR          & H@10        & MRR           & H@10 & MRR           & H@10          \\ 
\midrule
$\mathbb{E}^k$    & .496         & .585        & .348              &     .534       & .565   & .703   \\
$\mathbb{P}^k$    & .505         & .593        & .358              &     .547       & .572   & .708   \\
$\mathbb{Q}^k$  & .513         & .606        & .354              & .542        & .580   & .714       \\
\midrule
$(\mathbb{P}_{i}^{k_\mathbb{P}})^{m_\mathbb{P}}$  & .509         & .602        & \underline{.365}              & \underline{.553}        & .578   & .712       \\
$(\mathbb{Q}_{j}^{k_\mathbb{Q}})^{m_\mathbb{Q}}$ & \underline{.518}         & \underline{.610}        & .360              & .549        & \underline{.585}   & \underline{.718}      \\
\midrule
$(\mathbb{P}_{i}^{k_\mathbb{P}})^{m_\mathbb{P}}$$\times$$(\mathbb{Q}_{j}^{k_\mathbb{Q}})^{m_\mathbb{Q}}$ & \textbf{.525}         & \textbf{.615}        & \textbf{.370}              & \textbf{.558}        & \textbf{.588}   & \textbf{.723}      \\
\bottomrule
\end{tabular}}
\end{small}
\end{center}
\vspace{-5mm}
\end{table}

\subsection{Performance Comparison}
\label{main}
Table~\ref{WN18RR-and-FB15k237} summarizes the main results on link prediction task.
One can observe that GoldE framework consistently outper-forms the baselines on all metrics across the three datasets.
In particular, 
compared to HousE, the most relevant baseline that
represents relations as $k$-dimensional Euclidean House-holder rotations, 
GoldE achieves an average absolute gain of $2.5\%$ in MRR, and even surpasses its projection-enhanced variant (i.e., HousE-pro) with $1.3\%$ absolute improvement in MRR.
Moreover, GoldE also outperforms exisiting hyperbolic orthogonal models (i.e., RefH, RotH and AttH) by a clear margin.
Such new state-of-the-art results demonstrate the superiority of our universal orthogonal parameterization.

\begin{figure}[t!]
\centering
\vspace{3mm}
\subfigure[MRR vs. $k_\mathbb{X}$ on WN18RR]{
            \begin{tikzpicture}[font=\Large, scale=0.45]
                \begin{axis}[
                    width=0.46\textwidth, 
                    height=0.35\textwidth,
                    legend cell align={left},
                    legend style={nodes={scale=0.88, transform shape}, legend columns=-1, /tikz/every even column/.append style={column sep=0.25cm}},
                    xlabel={Dimension $k_\mathbb{X}$},
                    xtick pos=left,
                    tick label style={font=\Large},
                    ylabel style={font=\Large, yshift=10pt},
                    ylabel={MRR},
                    xtick={2, 4, 6, 8, 10, 12},
                    xticklabels={$2$, $4$, $6$, $8$, $10$, $12$},
                    ytick={0.470, 0.482,0.494,0.506,0.518},
                    yticklabels={$0.470$,$0.482$,$0.494$,$0.506$,$0.518$},
                    legend pos=south east,
                    ymajorgrids=true,
                    grid style=dashed
                ]
                \addplot[
                    color=brown,
                    dotted,
                    mark options={solid},
                    mark=triangle*,
                    line width=1.5pt,
                    mark size=2pt
                    ]
                    coordinates {
                    (2, 0.471)
                    (4, 0.489)
                    (6, 0.493)
                    (8, 0.494)
                    (10, 0.496)
                    (12, 0.496)
                    };
                    \addlegendentry{$\mathbb{E}$}
                \addplot[
                    color=blue,
                    dotted,
                    mark options={solid},
                    mark=*,
                    line width=1.5pt,
                    mark size=2pt
                    ]
                    coordinates {
                    (2, 0.485)
                    (4, 0.500)
                    (6, 0.506)
                    (8, 0.508)
                    (10, 0.509)
                    (12, 0.509)
                    };
                    \addlegendentry{$\mathbb{P}$}
                \addplot[
                    color=purple,
                    dotted,
                    mark options={solid},
                    mark=diamond*,
                    line width=1.5pt,
                    mark size=2pt
                    ]
                    coordinates {
                    (2, 0.496)
                    (4, 0.513)
                    (6, 0.516)
                    (8, 0.518)
                    (10, 0.518)
                    (12, 0.518)
                    };
                    \addlegendentry{$\mathbb{Q}$}
                \end{axis}
                \end{tikzpicture}
    }\hfill
    \subfigure[MRR vs. $k_\mathbb{X}$ on FB15k-237]{
            \begin{tikzpicture}[font=\Large,scale=0.45]
                \begin{axis}[
                    width=0.46\textwidth, 
                    height=0.35\textwidth,
                    legend cell align={left},
                    legend style={nodes={scale=0.88, transform shape}, legend columns=-1, /tikz/every even column/.append style={column sep=0.25cm}},
                    xlabel={Dimension $k_\mathbb{X}$},
                    xtick pos=left,
                    tick label style={font=\Large},
                    ylabel style={font=\Large, yshift=10pt},
                    ylabel={MRR},
                    xtick={2, 6, 10, 14, 18, 22},
                    xticklabels={$2$, $6$, $10$, $14$, $18$, $22$},
                    ytick={0.337, 0.344,0.351,0.358,0.365},
                    yticklabels={$0.337$,$0.344$,$0.351$,$0.358$,$0.365$},
                    legend pos=south east,
                    ymajorgrids=true,
                    grid style=dashed
                ]
                \addplot[
                    color=brown,
                    dotted,
                    mark options={solid},
                    mark=triangle*,
                    line width=1.5pt,
                    mark size=2pt
                    ]
                    coordinates {
                    (2, 0.338)
                    (6, 0.340)
                    (10, 0.345)
                    (14, 0.347)
                    (18, 0.348)
                    (22, 0.348)
                    };
                    \addlegendentry{$\mathbb{E}$}
                \addplot[
                    color=blue,
                    dotted,
                    mark options={solid},
                    mark=*,
                    line width=1.5pt,
                    mark size=2pt
                    ]
                    coordinates {
                    (2, 0.351)
                    (6, 0.355)
                    (10, 0.360)
                    (14, 0.363)
                    (18, 0.365)
                    (22, 0.365)
                    };
                    \addlegendentry{$\mathbb{P}$}
                \addplot[
                    color=purple,
                    dotted,
                    mark options={solid},
                    mark=diamond*,
                    line width=1.5pt,
                    mark size=2pt
                    ]
                    coordinates {
                    (2, 0.347)
                    (6, 0.350)
                    (10, 0.356)
                    (14, 0.359)
                    (18, 0.360)
                    (22, 0.360)
                    };
                    \addlegendentry{$\mathbb{Q}$}
                \end{axis}
                \end{tikzpicture}
    }
    \vspace{-3.0mm}
    \caption{MRR results of GoldE under the product of $\mathbb{X}$ with varying dimension $k_\mathbb{X}$ $(\mathbb{X}\in\{\mathbb{E},\mathbb{P},\mathbb{Q}\})$ on WN18RR and FB15k-237.}
    \label{exp-fig}
    \vspace{-5mm}
\end{figure}
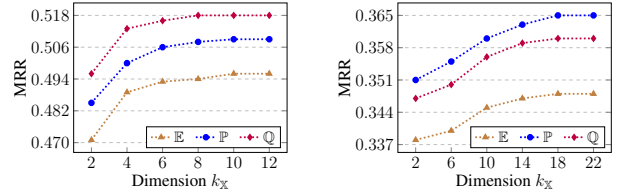

\subsection{Effect of Universal Orthogonal Parameterization}
\label{analysis}

\textbf{Effect of geometric unification:}
To investigate the effect of geometries for orthogonal relation transformations, we explore the performance of GoldE under different geometric combinations\footnote[4]{To simplify, we use a shorthand notation for repeated components: $(\mathbb{X}_i^{k_\mathbb{X}})^{m_\mathbb{X}}$$=$$\bigtimes_{i=1}^{m_\mathbb{X}}\mathbb{X}_{i}^{k_\mathbb{X}}$. Note that the Cartesian product of $\mathbb{E}$ is redundant since it satisfies $\mathbb{E}^k=\bigtimes_{i=1}^{m_\mathbb{E}}\mathbb{E}_{i}^{k_\mathbb{E}}$, while this equality does not hold for non-Euclidean manifolds \cite{mix2}.}.
From the results summarized in Table~\ref{geometry-ablation}, one can observe the following: 
(1) Both elliptic and hyperbolic parameterizations are more effective than the Euclidean counterpart, indicating the importance of additional scaling operations and modeling inherent hierarchies, respectively; 
(2) As we expected, hyperbolic parameterization performs better on the two datasets (i.e., WN18RR and YAGO3-10) dominated by hierarchies, while the elliptic counterpart excels on the dataset (i.e., FB15k-237) dominated by cycles.
Note that the product of multiple $\mathbb{P}$ (or $\mathbb{Q}$) outperforms its single case, since such homogeneous combination is able to embed a variety of cycles (or hierarchies) \cite{mix1};
(3) By combining multiple $\mathbb{P}$ and $\mathbb{Q}$, our GoldE can naturally preserve both cyclical and hierarchical structures, thereby achieving the optimal performance on all the three datasets.
Such improvements confirm the effectiveness of the geometric unification in our universal orthogonal parameterization.

\textbf{Effect of dimensional extension:}
To verify the expressiveness of the dimensional extension in our parameterization, we further conduct experiments for GoldE under the product of $\mathbb{X}$ with varying dimension $k_\mathbb{X}$. 
Figure~\ref{exp-fig} exhibits the MRR results on WN18RR and FB15k-237.
As expected, whether in Euclidean, elliptic, or hyperbolic parameterizations, 
the higher-dimensional orthogonal transformations enhance modeling capability, leading to superior performance compared to their lower-dimensional counterparts. 

\begin{table}[t]
\caption{MRR for the models tested on each relation of WN18RR.}
\label{Case-WN18RR}
\vspace{-0.5mm}
\begin{center}
\begin{small}
\resizebox{0.93\columnwidth}{!}{
\begin{tabular}{lcccc}
\toprule
Relation Name   & \#Triple  & RotH  & HousE  & GoldE  \\
\midrule
hypernym         &  1,251 & \underline{0.190} & 0.182    & \textbf{0.226} \\
instance\_hypernym   &  122 & 0.362 & \underline{0.395}    & \textbf{0.453} \\
member\_meronym      & 253 & 0.266 & \underline{0.275}    & \textbf{0.296} \\
synset\_domain\_topic\_of    &  114 & \underline{0.406} & 0.396    & \textbf{0.430} \\
has\_part            & 172 & 0.213 & \underline{0.217}    & \textbf{0.255} \\
member\_of\_domain\_usage     & 24 & 0.350 & \underline{0.415}    & \textbf{0.455} \\
member\_of\_domain\_region    & 26 & \underline{0.356} & 0.281    & \textbf{0.403} \\
derivationally\_related\_form  & 1,074 & 0.947 & \underline{0.958} & \textbf{0.961}    \\
also\_see           & 56 & 0.583 & \underline{0.638} & \textbf{0.665}    \\
verb\_group                  & 39 & 0.820 & \underline{0.968}    & \textbf{0.975} \\
similar\_to             & 3 & \textbf{1.000} & \textbf{1.000}    & \textbf{1.000} \\
\bottomrule
\end{tabular}}
\end{small}
\end{center}
\vspace{-6mm}
\end{table}

\subsection{Fine-grained Performance Analysis}
\label{fine-grained}
To verify the modeling capability of our GoldE from a fine-grained perspective, we report its performance per relation on WN18RR in Table~\ref{Case-WN18RR}.
Compared to the hyperbolic model RotH and the Euclidean model HousE, our GoldE achieves the best performance on all 11 relation types.
Specifically, GoldE not only obtains significant improvements on hierarchical relations (e.g., \emph{hypernym} and \emph{instance\_hypernym}), but also performs well on symmetric relations (e.g., \emph{derivationally\_related\_form} and \emph{verb\_group}).
These advanced fine-grained results demonstrate the superior and comprehensive modeling  capability of our GoldE framework.

\subsection{Results with Restricted Embedding Size}
\label{restricted_k}
We also follow \cite{RotH} to evaluate our GoldE under the restriction of fixing the embedding size $k$ to $32$.
It is challenging for KGE models to offer qualified representations in such a restricted setting.
Table \ref{low-setting} reports the experimental results on the three datasets.
We observe that GoldE consistently surpasses the baselines by a clear margin across all datasets, confirming the higher-quality representations learned by our universal orthogonal parameterization.
See Appendix \ref{app_curve} for more results of GoldE with varying $k$.

\section{Related Work}
Early KGE approaches explore various geometric operations such as translation~\cite{TransE, transh, transr, transd} and scaling~\cite{distmult,complex} to model relations.
As a representative work, RotatE~\cite{sun2019rotate} goes beyond
the previous methods
by modeling relations as $2$-dimensional rotations between entities to handle the crucial logical patterns in KGs.
Inspired by the effectiveness of orthogonal transformations, subsequent research mainly advances in two directions.

\textbf{Dimension direction:} A series of works~\cite{quate,rotate3d,duale,house} seek higher-dimensional orthogonal transformations for better modeling capacity. 
Recently, HousE~\cite{house} introduces the Euclidean Householder reflections to represent relations as $k$-dimensional Euclidean rotations, which can be viewed as the Euclidean specialization of our proposal.

\textbf{Geometry direction:} Another branch of KGE models~\cite{MuRP,RotH} exploits hyperbolic geometry for preserving the hierarchical structures in KGs.
AttH~\cite{RotH} represents relations as hyperbolic orthogonal transformations to accommodate both hierarchical structures and logical patterns, yet remains constrained by low-dimensional space and homogeneous geometry.

Notably, our GoldE framework generalizes existing KGE models in terms of both dimension and geometry.
Moreover, the designed parameterization is orthogonal to advanced techniques 
such as regularizers \cite{N3} and graph encoders \cite{COMPGCN,NBFNet}, allowing for their integration to achieve better performance.

\begin{table}[t]
\caption{Link prediction results for $32$-dimensional embeddings. 
}
\label{low-setting}
\vspace{-0.5mm}
\begin{center}
\begin{small}
\resizebox{0.93\columnwidth}{!}{
\begin{tabular}{lcccccc}
\toprule
         & \multicolumn{2}{c}{WN18RR} & \multicolumn{2}{c}{FB15k-237} & \multicolumn{2}{c}{YAGO3-10} \\ \cmidrule(r){2-3} \cmidrule(r){4-5} \cmidrule(r){6-7}
Model & MRR          & H@10        & MRR           & H@10 & MRR           & H@10          \\ 
\midrule
RotatE    & .387         & .491        & .290              &     .458       & -   & -   \\
QuatE    & .421         & .467        & .293              &     .460       & -   & -   \\
MuRE  & .458         & .525        & .313              & .489        & .283   & .478       \\
MuRP  & .465         & .544        & .323              & \underline{.501}        & .230   & .392       \\
RefH & .447         & .518        & .312              & .489        & .381   & .530      \\
RotH & \underline{.472}         & \underline{.553}        & .314              & .497        & .393   & .559      \\
AttH & .466         & .551        & \underline{.324}              & \underline{.501}        & \underline{.397}   & \underline{.566}      \\
\midrule
GoldE    & \textbf{.482}         & \textbf{.558}        & \textbf{.335}              & \textbf{.506}        & \textbf{.420}   & \textbf{.586}      \\
\bottomrule
\end{tabular}}
\end{small}
\end{center}
\vspace{-8mm}
\end{table}

\section{Conclusion}
In this paper, we introduce GoldE,
a powerful and general framework that generalizes existing KGE approaches based on a universal orthogonal parameterization.
Empowered by such parameterization, our framework can simultaneously achieve dimensional extension and geometric unification for relational orthogonalization, thereby possessing the superior capability of capturing logical patterns and heterogeneous structures in KGs.
Experimental results over three datasets comprehensively validate the effectiveness of our proposal.
In future work, we will explore the potential of GoldE in modeling hyper-relational and temporal knowledge graphs.

\section*{Acknowledgements}

This work is supported in part by National Natural Science Foundation of China (No. 62102420), Beijing Outstanding Young Scientist Program NO. BJJWZYJH 012019100020098, Intelligent Social Governance Platform, Major Innovation \& Planning Interdisciplinary Platform for the ``Double-First Class'' Initiative, Renmin University of China, Public Computing Cloud, Renmin University of China, fund for building world-class universities (disciplines) of Renmin University of China, Intelligent Social Governance Platform.


\section*{Impact Statement}

This work presents GoldE based on a universal orthogonal parameterization, aiming to advance the field of knowledge graph embedding in terms of both dimension and geometry.
We highlight that the designed orthogonal parameterization is a general technique, even with the potential to enhance various foundational models such as RNN \cite{orth-rnn}, GNN \cite{orth-gnn} and Transformer \cite{orth-transformer}, thus possessing the capability of facilitating a wide range of downstream tasks.
As for the negative impact, GoldE may encode the biases present in the training graph, which can lead to stereotyped predictions for missing links between certain nodes (e.g., users on a social or e-commerce platform).
We expect future studies to mitigate this issue.


\bibliography{example_paper}
\bibliographystyle{icml2024}

\newpage
\appendix
\onecolumn


\section{Glossary of Symbols}
\label{notations}
\begin{table}[h]
\vspace{-3mm}
\caption{Glossary of variables and symbols used in this paper.}
\label{symbols}
\vspace{2mm}
\begin{center}
\begin{small}
\renewcommand{\arraystretch}{1.5}
\begin{tabular}{ccl}
\toprule
Symbol     & Shape & Description  \\
\midrule
$\mathcal{V}$ & - & Set of entities \\
$\mathcal{R}$ & - &  Set of relations \\
$\mathcal{F}$ & - &  Set of factual triples \\
$h, t$  & -  & Head entity and tail entity \\
$r$     & -    & Relation type \\
$\mathbb{E}^k$   &   -      &  $k$-dimensional Euclidean space, equivalent to $\mathbb{R}^k$\\
$\mathbb{P}^k$   &   -      &  Elliptic subspace embedded in $\mathbb{R}^k$\\
$\mathbb{Q}^k$   &   -      &  Hyperbolic subspace embedded in $\mathbb{R}^k$\\
$\mathbb{D}^k$   &   -      &  Product manifold with total dimension $k$\\
$\mathbb{X}_i^{k_i}$   &   -      &  The $i$-th component space of product manifold $\mathbb{D}$ \\
$m$   &   $\mathbb{R}_{+}$      &  Number of component spaces\\
$\beta_r$     &  $\mathbb{R}_{+}$      &  Param. of $r\in \mathcal{R}$ that determines the curvature of manifold\\
$\mathbf{p}_r$     &  $\mathbb{R}_{+}^{k}$    & $r$-specific elliptic weighting vector \\
$\mathbf{q}$ & $\mathbb{R}^k$ & Hyperbolic weighting vector $(-1, 1, \ldots, 1)$ \\
$\mathbf{e}_v$  & $\mathbb{R}^k$  & Embedding of entity 
$v\in\mathcal{V}$ \\
$\mathbf{e}_{v,r}$  & $\mathbb{R}^k$  & $r$-associated representation of entity 
$v\in\mathcal{V}$ \\
$\mathbf{b}_r$     &  $\mathbb{R}^{k}$    & Param. of $r\in \mathcal{R}$ for hyperbolic boost \\
$\mathbf{U}_r$     &  $\mathbb{R}^{k \times k}$    & Param. of $r\in \mathcal{R}$ for generalized orthogonal transformation \\

\bottomrule
\end{tabular}
\end{small}
\end{center}
\vskip -0.1in
\end{table}

\section{Proof of Theorem~\ref{thm-1}}
\label{proof-thm-1}
This theorem is a special case of the ``Cartan-Dieudonn\'e Theorem"~\cite{cartan2012theory,dieudonne1958groupes}, which is one of the fundamental results of geometric algebra~\cite{gallier2011geometric}.
Since this result is not used in the remainder of this paper and the proof is somewhat intricate, we encourage the beginning reader to read the statement and then skip to the next part.
\subsection{Proof of Theorem~\ref{CD} (Cartan-Dieudonn\'e Theorem)}
\begin{theorem}
\label{CD}
{\rm{(Cartan-Dieudonn\'e Theorem).}} In the non-degenerate quadratic space $V$ endowed with $\braket{\cdot,\cdot}_\mathbf{w}$, every element $\mathbf{G}$ of the orthogonal group $\mathbf{O}_\mathbf{w}(k)$ can be expressed as a product of at most $k$ reflections.
\end{theorem}
\begin{proof}
(1) The first step is to prove the following lemma: \textit{Suppose that $\mathbf{G}\in\mathbf{O}_\mathbf{w}(k)$ satisfies such a condition: for every anisotropic vector $\mathbf{x}$, the vector $\mathbf{Gx}-\mathbf{x}$ is nonzero and isotropic\footnote[5]{A vector $\mathbf{x}\in V$ is said to be \textbf{isotropic} if $\braket{\mathbf{x},\mathbf{x}}_\mathbf{w}=0$ and \textbf{anisotropic} otherwise.}. Then, $k\geq4$, $k$ is even, and $\mathbf{G}$ is a rotation.}

It is evident that we cannot have $k=1$.
If $k=2$, let $\mathbf{x}$ be an anisotropic vector. 
Since $\mathbf{Gx}-\mathbf{x}$ is isotropic and nonzero, $\mathbf{Gx}$ must be linearly independent from $\mathbf{x}$.
This implies that the determinant of the quadratic form with respect to the basis $\mathbf{Gx}-\mathbf{x}$ and $\mathbf{x}$ is equal to $0$, contradicting the non-degeneracy of the space.
We therefore must have $k\geq3$.

Since $\mathbf{Gx}-\mathbf{x}$ is isotropic for all anisotropic $\mathbf{x}$, we have $q(\mathbf{Gx}-\mathbf{x})=\braket{\mathbf{Gx}-\mathbf{x},\mathbf{Gx}-\mathbf{x}}_\mathbf{w}=0$.
In fact, this holds for all $\mathbf{x}$. To see this, let $\mathbf{y}$ be a nonzero isotropic vector. 
There is a plane containing $\mathbf{y}$ and splitting the space $V$~\cite{o2013introduction}, hence there is a vector $\mathbf{z}$ with $q(\mathbf{z})\neq0$ and $\braket{\mathbf{y},\mathbf{z}}_\mathbf{w}=0$. Then, we have $q(\mathbf{y}+\epsilon\mathbf{z})\neq0$ for any $\epsilon\in\mathbb{R}\hspace{-0.5mm}\setminus\hspace{-0.5mm}\{0\}$, hence 
\begin{equation}
q(\mathbf{G}(\mathbf{y}+\epsilon\mathbf{z})-(\mathbf{y}+\epsilon\mathbf{z}))=0, \quad q(\mathbf{G}\mathbf{z}-\mathbf{z})=0.
\nonumber
\end{equation}
It follows that
\begin{equation}
q(\mathbf{G}\mathbf{y}-\mathbf{y})+2\epsilon\braket{\mathbf{G}\mathbf{y}-\mathbf{y},\mathbf{G}\mathbf{z}-\mathbf{z}}_\mathbf{w}=0.
\nonumber
\end{equation}
If we substitute $\epsilon=1$, and then $\epsilon=-1$, and then add, we obtain $q(\mathbf{G}\mathbf{y}-\mathbf{y})=0$ as we asserted.
Therefore, if we define the space $W:=(\mathbf{G}-1)V$, we have $q(\mathbf{y})=0$ for any $y\in W$, i.e., $q(W)=0$.
Then, the orthogonal complement of $W$ is $W^\perp=\{\mathbf{x}\in V \mid \forall\mathbf{y}\in W, \braket{\mathbf{x},\mathbf{y}}_\mathbf{w} = 0\}$.
Now for any $\mathbf{x}\in V$ and $\mathbf{y}\in W^\perp$, we have 
\begin{equation}
\begin{split}
\braket{\mathbf{x},\mathbf{G}\mathbf{y}-\mathbf{y}}_\mathbf{w}&=\braket{\mathbf{G}\mathbf{x},\mathbf{G}\mathbf{y}-\mathbf{y}}_\mathbf{w}-\braket{\mathbf{G}\mathbf{x}-\mathbf{x},\mathbf{G}\mathbf{y}-\mathbf{y}}_\mathbf{w}\\
&=\braket{\mathbf{G}\mathbf{x},\mathbf{G}\mathbf{y}-\mathbf{y}}_\mathbf{w}\\
&=\braket{\mathbf{G}\mathbf{x},\mathbf{G}\mathbf{y}}_\mathbf{w}-\braket{\mathbf{G}\mathbf{x},\mathbf{y}}_\mathbf{w}\\
&=\braket{\mathbf{x},\mathbf{y}}_\mathbf{w}-\braket{\mathbf{G}\mathbf{x},\mathbf{y}}_\mathbf{w}\\
&=-\braket{\mathbf{G}\mathbf{x}-\mathbf{x},\mathbf{y}}_\mathbf{w}\\
&=0.
\nonumber
\end{split}
\end{equation}
Hence $\mathbf{G}\mathbf{y}-\mathbf{y}$ is perpendicular to all of $V$. Since $V$ is non-degenerate, we conclude that $\mathbf{G}\mathbf{y}=\mathbf{y}$ for all $\mathbf{y}\in W^\perp$.
Then, $q(W^\perp)=0$ by the given condition. Thus we get $W=W^\perp$.
Therefore, $k=\mathrm{dim}(V)=\mathrm{dim}(W)+\mathrm{dim}(W^\perp)=2\mathrm{dim}(W)$ is even, hence at least $4$.
Moreover, since $\mathbf{G}$  acts as the identity on a maximal totally isotropic subspace of $V$~\cite{clark2013quadratic}, $\mathbf{G}$ is a rotation with  the determinant equal to $+1$, completing Step 1.

(2) Now we can prove the theorem. The proof is by induction on $k$.
For $k=1$, the result is trivial, so let $k>1$.

\textbf{Case 1:} Suppose there exists an anisotropic vector $\mathbf{x}\in V$ such that $\mathbf{G}\mathbf{x}=\mathbf{x}$.
Then the restriction of $\mathbf{G}$ to the hyperplane $\zeta$ orthogonal to $\mathbf{x}$ is an element of $\mathbf{O}_\mathbf{w}(k-1)$.
By the inductive assumption, this restriction is a product of at most $k-1$ reflections taken with respect to lines in $\zeta$. 
One can naturally view each of them as the reflection on all of $V$, and the same product of reflections agrees with $\mathbf{G}$ on $\zeta$.
Moreover, it also agrees with $\mathbf{G}$ on $\mathbf{x}$, since $\mathbf{G}$ and all of the reflections are equal to the identity on $\mathbf{x}$. Thus $\mathbf{G}$ is itself equal to the product of the at most $k-1$ reflections.

\textbf{Case 2:} Next suppose that there is an anisotropic vector $\mathbf{x}$ such that $\mathbf{G}\mathbf{x}-\mathbf{x}$ is anisotropic, i.e., $q(\mathbf{G}\mathbf{x}-\mathbf{x})\neq0$.
To facilitate presentation, we denote the elementary reflection in Equation (\ref{ref-new}) as $\tau_\mathbf{u}$.
In this way, we have
\begin{equation}
\begin{split}
\tau_{\mathbf{G}\mathbf{x}-\mathbf{x}}(\mathbf{Gx})&=\mathbf{Gx}-2\frac{\braket{\mathbf{Gx},\mathbf{Gx}-\mathbf{x}}_\mathbf{w}}{\braket{\mathbf{Gx}-\mathbf{x},\mathbf{Gx}-\mathbf{x}}_\mathbf{w}}(\mathbf{Gx}-\mathbf{x})\\
&=\mathbf{Gx}-2\frac{\braket{\mathbf{x},\mathbf{x}}_\mathbf{w}-\braket{\mathbf{G}\mathbf{x},\mathbf{x}}_\mathbf{w}}{\braket{\mathbf{Gx}-\mathbf{x},\mathbf{Gx}-\mathbf{x}}_\mathbf{w}}(\mathbf{Gx}-\mathbf{x})\\
&=\mathbf{Gx}-\frac{\braket{\mathbf{Gx}-\mathbf{x},\mathbf{Gx}-\mathbf{x}}_\mathbf{w}}{\braket{\mathbf{Gx}-\mathbf{x},\mathbf{Gx}-\mathbf{x}}_\mathbf{w}}(\mathbf{Gx}-\mathbf{x})\\
&=\mathbf{Gx}-(\mathbf{Gx}-\mathbf{x})\\
&=\mathbf{x},
\nonumber
\end{split}
\end{equation}
which implies that $\tau_{\mathbf{G}\mathbf{x}-\mathbf{x}}\mathbf{G}$ leaves $\mathbf{x}$ fixed.
Based on the first case, the orthogonal matrix $\tau_{\mathbf{G}\mathbf{x}-\mathbf{x}}\mathbf{G}$ is a product of at most $k-1$ reflections, hence $\mathbf{G}$ is a product of at most $k$ reflections.

\textbf{Case 3:} According to Case 1 and 2, any $\mathbf{G}$ that does not satisfy the condition in Step 1 is a product of at most $k$ reflections. 
Now, let's consider the $\mathbf{G}$ that for every anisotropic vector $\mathbf{x}\in V$, $\mathbf{Gx}-\mathbf{x}$ is isotropic and nonzero.
Based on Step 1, we conclude that $k$ is even and $\mathbf{G}$ is a rotation matrix.
Then, $\mathbf{G}'=\tau\mathbf{G}$ is a reflection that cannot satisfy the condition in Step 1.
Hence $\mathbf{G}'$ can be expressed as a product of at most $k$ reflections based on Case 1 and 2.
Therefore, $\mathbf{G}$ is itself a product of at most $k+1$ reflections.
However, $\mathbf{G}$ cannot be a product of $k+1$ reflections, since $k+1$ is odd and $\mathbf{G}$ is a rotation matrix. 
Hence $\mathbf{G}$ is a product of at most $k$ reflections, qed.
\end{proof}

\subsection{Proof of Theorem~\ref{thm-1}}
\begin{proof}
First, the designed mapping $\mathrm{Orth}$ is a product of $n$ elementary reflections, and its output is a generalized orthogonal matrix that satisfies the quadratic inner product invariance, i.e., $(\mathbf{H}_n\cdots\mathbf{H}_1)^\top\mathrm{diag}(\mathbf{w})(\mathbf{H}_n\cdots\mathbf{H}_1)=\mathrm{diag}(\mathbf{w})$. Therefore, we have $\mathrm{Image}(\mathrm{Orth})\subseteq\mathbf{O}_\mathbf{w}(k)$. 
Then, according to the Cartan-Dieudonn\'e Theorem, every generalized orthogonal matrix $\mathbf{G}\in\mathbf{O}_\mathbf{w}(k)$ can be expressed as $\prod_{c=1}^z\mathbf{H}(\mathbf{u}_c,\mathbf{w})$ for a certain $z\leq k$.
Since the $k\times k$ identity matrix $\mathbf{I}\in\mathrm{Image}(\mathrm{Orth})$ for any value of $n$, we can get $\mathbf{G}=\mathbf{G}\mathbf{I}=\prod_{c=1}^k\mathbf{H}(\mathbf{u}_c,\mathbf{w})$.
Therefore, we also have $\mathbf{O}_\mathbf{w}(k)\subseteq\mathrm{Image}(\mathrm{Orth})$ when $n=k$.
On the whole, we conclude that $\mathrm{Image}(\mathrm{Orth})=\mathbf{O}_\mathbf{w}(k)$.
\end{proof}

\section{Proof of Claim~\ref{pro1}}
\label{proof-pro-1}
\begin{proof}
We show that the elliptic parameterization is equivalent to Euclidean parameterization equipped with element-wise scaling transformations determined by $\sqrt{\mathbf{p}_r}$.
Specifically, the distance function of elliptic parameterization in Equation (\ref{ell_score}) can be reformulated as follows:
\begin{equation}
\begin{split}
d_{\mathbb{P}}(\mathbf{e}'_{h,r},\mathbf{e}_{t,r})&=\sqrt{(\mathbf{e}'_{h,r}-\mathbf{e}_{t,r})^\top\mathrm{diag}(\mathbf{p}_r)(\mathbf{e}'_{h,r}-\mathbf{e}_{t,r})}\\
&=\sqrt{(\mathrm{diag}(\sqrt{\mathbf{p}_r})\mathbf{e}'_{h,r}-\mathrm{diag}(\sqrt{\mathbf{p}_r})\mathbf{e}_{t,r})^\top(\mathrm{diag}(\sqrt{\mathbf{p}_r})\mathbf{e}'_{h,r}-\mathrm{diag}(\sqrt{\mathbf{p}_r})\mathbf{e}_{t,r})}\\
&=\Vert\mathrm{diag}(\sqrt{\mathbf{p}_r})\mathbf{e}'_{h,r}-\mathrm{diag}(\sqrt{\mathbf{p}_r})\mathbf{e}_{t,r}\Vert\\
&=\Vert\mathrm{diag}(\sqrt{\mathbf{p}_r})\prod_{i=1}^k \mathbf{H}(\mathbf{U}_r[i],\mathbf{p}_r)\mathbf{e}_{h,r}-\mathrm{diag}(\sqrt{\mathbf{p}_r})\mathbf{e}_{t,r}\Vert\\
&=\Vert\mathrm{diag}(\sqrt{\mathbf{p}_r})(\mathbf{I}-2\frac{\mathbf{U}[k]\mathbf{U}[k]^\top\mathrm{diag}(\mathbf{p}_r)}{\mathbf{u}^\top\mathrm{diag}(\mathbf{p}_r)\mathbf{u}})\prod_{i=1}^{k-1} \mathbf{H}(\mathbf{U}_r[i],\mathbf{p}_r)\mathbf{e}_{h,r}-\mathrm{diag}(\sqrt{\mathbf{p}_r})\mathbf{e}_{t,r}\Vert\\
&=\Vert(\mathbf{I}-2\frac{\mathrm{diag}(\sqrt{\mathbf{p}_r})\mathbf{U}[k]\mathbf{U}[k]^\top\mathrm{diag}(\sqrt{\mathbf{p}_r})^\top}{\mathbf{U}[k]^\top\mathrm{diag}(\sqrt{\mathbf{p}_r})^\top\mathrm{diag}(\sqrt{\mathbf{p}_r})\mathbf{U}[k]})\mathrm{diag}(\sqrt{\mathbf{p}_r})\prod_{i=1}^{k-1} \mathbf{H}(\mathbf{U}_r[i],\mathbf{p}_r)\mathbf{e}_{h,r}-\mathrm{diag}(\sqrt{\mathbf{p}_r})\mathbf{e}_{t,r}\Vert\\
&=\Vert\mathbf{H}(\mathrm{diag}(\sqrt{\mathbf{p}_r})\mathbf{U}_r[k],\mathbf{1})\mathrm{diag}(\sqrt{\mathbf{p}_r})\prod_{i=1}^{k-1} \mathbf{H}(\mathbf{U}_r[i],\mathbf{p}_r)\mathbf{e}_{h,r}-\mathrm{diag}(\sqrt{\mathbf{p}_r})\mathbf{e}_{t,r}\Vert\\
&=\Vert\prod_{i=1}^{k} \mathbf{H}(\mathrm{diag}(\sqrt{\mathbf{p}_r})\mathbf{U}_r[i],\mathbf{1})\mathrm{diag}(\sqrt{\mathbf{p}_r})\mathbf{e}_{h}-\mathrm{diag}(\sqrt{\mathbf{p}_r})\mathbf{e}_{t}\Vert\\
&=\Vert\mathrm{Orth}(\mathrm{diag}(\sqrt{\mathbf{p}_r})\mathbf{U}_r,\mathbf{1})\mathrm{diag}(\sqrt{\mathbf{p}_r})\mathbf{e}_{h}-\mathrm{diag}(\sqrt{\mathbf{p}_r})\mathbf{e}_{t}\Vert.
\nonumber
\end{split}
\end{equation}
This expression implies that the head embedding $\mathbf{e}_h\in\mathbb{R}^k$ is first scaled by $\sqrt{\mathbf{p}_r}$, and then performs a Euclidean orthogonal transformation, with the expectation of being similar to the tail embedding $\mathbf{e}_t\in\mathbb{R}^k$, which is also scaled by $\sqrt{\mathbf{p}_r}$.
\end{proof}

\section{Proof of Proposition~\ref{pro2}}
\label{proof-pro-2}
\begin{proof}
For any $\mathbf{G}\in\mathbf{O}_\mathbf{q}(k)$ and $\mathbf{x}\in\mathbb{Q}_\beta^k$, the hyperbolic orthogonal transformation $\mathbf{Gx}$ can be expressed in the following block form:
\begin{equation}
\begin{split}
\mathbf{Gx} = \begin{bmatrix}
  a & \mathbf{b}^\top\\ 
  \mathbf{c} & \mathbf{A}
\end{bmatrix}\begin{bmatrix}
  x_1\\ 
  \mathbf{x}_{2:k}
\end{bmatrix}=
\begin{bmatrix}
  ax_1 + \mathbf{b}^\top\mathbf{x}_{2:k}\\ 
  x_1\mathbf{c} + \mathbf{A}\mathbf{x}_{2:k}
\end{bmatrix},\nonumber
\end{split}
\end{equation}
where $a\in\mathbb{R}$, $\mathbf{b},\mathbf{c}\in\mathbb{R}^{k-1}$ and $\mathbf{A}\in\mathbb{R}^{(k-1)\times(k-1)}$.
We prove that (1) the absolute value of the first element of $\mathbf{G}$ is greater than or equal to $1$, i.e., $|G_{11}|=|a|\geq1$; (2) when $a\geq1$ (or $a\leq-1$), the transformation $\mathbf{Gx}$ preserves (or reverses) the sign of the first element of $\mathbf{x}$, i.e., $\mathrm{sign}(x_1)=\mathrm{sign}(ax_1 + \mathbf{b}^\top\mathbf{x}_{2:k})$ for $a\geq1$ and $\mathrm{sign}(x_1)\neq\mathrm{sign}(ax_1 + \mathbf{b}^\top\mathbf{x}_{2:k})$ for $a\leq-1$.

According to the definition of generalized orthogonal matrix in Equation (\ref{inner-inv}), one can verify that the hyperbolic orthogonal matrix $\mathbf{G}\in\mathbf{O}_\mathbf{q}(k)$ satisfies the inner product invariance:
\begin{equation}
    \mathbf{G}^\top\mathrm{diag}(\mathbf{q})\mathbf{G}=\mathrm{diag}(\mathbf{q}) \Rightarrow \mathbf{G}\mathrm{diag}(\mathbf{q})\mathbf{G}^\top=\mathrm{diag}(\mathbf{q}).\nonumber
\end{equation}
Since $\mathbf{q}=(-1,1,\ldots,1)\in\mathbb{R}^k$, the first element of $\mathbf{G}$ (i.e., $G_{11}=a$) requires that
\begin{equation}
    a^2-\Vert\mathbf{b}\Vert^2=1 \Rightarrow
    |a|=\sqrt{1+\Vert\mathbf{b}\Vert^2} \geq 1.\nonumber
\end{equation}
We then consider the following square equation:
\begin{equation}
\begin{split}
    (\frac{\mathbf{b}}{a}+\frac{\mathbf{x}_{2:k}}{x_1})^2&=(\frac{\mathbf{b}}{a})^2+2\frac{\mathbf{b}^\top\mathbf{x}_{2:k}}{ax_1}+(\frac{\mathbf{x}_{2:k}}{x_1})^2\\&=(\frac{\Vert\mathbf{b}\Vert^2}{a^2})+2\frac{\mathbf{b}^\top\mathbf{x}_{2:k}}{ax_1}+(\frac{\Vert\mathbf{x}_{2:k}\Vert^2}{x_1^2})
    \\&=(\frac{a^2-1}{a^2})+2\frac{\mathbf{b}^\top\mathbf{x}_{2:k}}{ax_1}+(\frac{\Vert\mathbf{x}_{2:k}\Vert^2}{x_1^2})\geq 0.\nonumber
\end{split}
\end{equation}
Since $\mathbf{x}\in\mathbb{Q}_\beta^k$, one can easily derive that $x_1^2>\Vert\mathbf{x}_{2:k}\Vert^2$ according to the definition of  hyperboloid in Equation (\ref{hyp_define}).
Based on this, we further have the following inequality:
\begin{equation}
    -2\frac{\mathbf{b}^\top\mathbf{x}_{2:k}}{ax_1}\leq(\frac{a^2-1}{a^2})+(\frac{\Vert\mathbf{x}_{2:k}\Vert^2}{x_1^2})<(\frac{a^2-1}{a^2})+1<2.\nonumber
\end{equation}
When $a\geq+1$ and $x_1>0$, we have 
\begin{equation}
   2ax_1>-2\mathbf{b}^\top\mathbf{x}_{2:k} \Rightarrow ax_1+\mathbf{b}^\top\mathbf{x}_{2:k}>0.\nonumber
\end{equation}
When $a\geq+1$ and $x_1<0$, we have 
\begin{equation}
   2ax_1<-2\mathbf{b}^\top\mathbf{x}_{2:k} \Rightarrow ax_1+\mathbf{b}^\top\mathbf{x}_{2:k}<0.\nonumber
\end{equation}
When $a\leq-1$ and $x_1>0$, we have 
\begin{equation}
   2ax_1<-2\mathbf{b}^\top\mathbf{x}_{2:k} \Rightarrow ax_1+\mathbf{b}^\top\mathbf{x}_{2:k}<0.\nonumber
\end{equation}
When $a\leq-1$ and $x_1<0$, we have 
\begin{equation}
   2ax_1>-2\mathbf{b}^\top\mathbf{x}_{2:k} \Rightarrow ax_1+\mathbf{b}^\top\mathbf{x}_{2:k}>0.\nonumber
\end{equation}
Therefore, we can observe that the hyperbolic orthogonal matrix $\mathbf{G}$ with its first element $G_{11}=a\geq+1$ preserves the sign of $x_1$, while $\mathbf{G}$ with $G_{11}\leq-1$ reverses the sign of $x_1$.
Moreover, the set of all $\mathbf{G}$ with $G_{11}\geq+1$, namely the positive subset $\mathbf{O}^+_{\mathbf{q}}(k)=\{\mathbf{G}\in\mathbf{O}_{\mathbf{q}}(k): G_{11}\geq+1\}$, forms a multiplicative subgroup of $\mathbf{O}_{\mathbf{q}}(k)$, since (1) for any $\mathbf{G}_1,\mathbf{G}_2\in\mathbf{O}^+_{\mathbf{q}}(k)$, their multiplication also preserves the quadratic inner product of $\mathbf{x}$ and the sign of $x_1$, i.e., $\mathbf{G}_1\mathbf{G}_2\in\mathbf{O}^+_{\mathbf{q}}(k)$; (2) $\mathbf{O}^+_{\mathbf{q}}(k)$ contains the identity matrix $\mathbf{I}$; (3) every matrix $\mathbf{G}\in\mathbf{O}^+_{\mathbf{q}}(k)$ has an inverse $\mathbf{G}^{-1}\in\mathbf{O}^+_{\mathbf{q}}(k)$ such that $\mathbf{G}\mathbf{G}^{-1}=\mathbf{I}$.
\end{proof}

\section{Proof of Proposition~\ref{pos-map}}
\label{proof-pro-3}
\begin{proof}
The positive hyperbolic orthogonal matrix $\mathbf{G}\in\mathbf{O}^+_\mathbf{q}(k)$ can be written in block form as follows:
\begin{equation}
\mathbf{G} = \begin{bmatrix}
  a & \mathbf{b}^\top\\ 
  \mathbf{c} & \mathbf{A}
\end{bmatrix},
\nonumber
\end{equation}
where $a\geq+1$, $\mathbf{b},\mathbf{c}\in\mathbb{R}^{k-1}$ and $\mathbf{A}\in\mathbb{R}^{(k-1)\times(k-1)}$. Since $\mathbf{G}^\top\mathrm{diag}(\mathbf{q})\mathbf{G}=\mathbf{G}\mathrm{diag}(\mathbf{q})\mathbf{G}^\top=\mathrm{diag}(\mathbf{q})$, we get
\begin{equation}
    \begin{bmatrix}
  a & \mathbf{c}^\top\\ 
  \mathbf{b} & \mathbf{A}^\top
\end{bmatrix}\begin{bmatrix}
  -a & -\mathbf{b}^\top\\ 
  \mathbf{c} & \mathbf{A}
\end{bmatrix}=\begin{bmatrix}
  a & \mathbf{b}^\top\\ 
  \mathbf{c} & \mathbf{A}
\end{bmatrix}\begin{bmatrix}
  -a & -\mathbf{c}^\top\\ 
  \mathbf{b} & \mathbf{A}^\top
\end{bmatrix}=\begin{bmatrix}
  -1 & 0\\ 
  0 & \mathbf{I}
\end{bmatrix},
\nonumber
\end{equation}
then
\begin{equation}
\begin{split}
\mathbf{A}^\top\mathbf{A}&=\mathbf{I}+\mathbf{b}\mathbf{b}^\top,\\
\mathbf{c}^\top\mathbf{c}&=a^2-1,\\
\mathbf{A}^\top\mathbf{c}&=a\mathbf{b},\\
\mathbf{b}^\top\mathbf{b}&=a^2-1,\\
\mathbf{A}\mathbf{b}&=a\mathbf{c}.
\nonumber
\end{split}
\end{equation}
From $\mathbf{A}^\top\mathbf{A}=\mathbf{I}+\mathbf{b}\mathbf{b}^\top$, we get that $\mathbf{A}^\top\mathbf{A}$ is symmetric and positive definite.
Geometrically, it is well known that $\frac{\mathbf{b}\mathbf{b}^\top}{\mathbf{b}^\top\mathbf{b}}$ is the orthogonal projection onto the line determined by $\mathbf{b}$.
Consequently, $\mathbf{b}\mathbf{b}^\top$ has the eigenvalue $0$ with multiplicity $k-2$ and the eigenvalue $\mathbf{b}^\top\mathbf{b}=a^2-1$ with multiplicity $1$.
The eigenvectors associated with $0$ are orthogonal to $\mathbf{b}$ and the eigenvectors associated with $a^2-1$ are proportional with $\mathbf{b}$. It follows that $\mathbf{I}+\mathbf{b}\mathbf{b}^\top$ has the eigenvalue $1$ with multiplicity $k-2$ and the eigenvalue $a^2$ with multiplicity $1$, the eigenvectors being as before. 
Now, $\mathbf{A}$ has the polar form $\mathbf{A}=\mathbf{Q}\mathbf{S}$, where $\mathbf{Q}$ is a Euclidean orthogonal matrix, $\mathbf{S}$ is a positive semi-definite symmetric matrix and $\mathbf{S}^2=\mathbf{S}^\top\mathbf{S}=\mathbf{A}^\top\mathbf{A}=\mathbf{I}+\mathbf{b}\mathbf{b}^\top$.
Since $a\geq+1$, $\mathbf{S}=\sqrt{\mathbf{I}+\mathbf{b}\mathbf{b}^\top}$ has the eigenvalue $1$ with multiplicity $k-2$ and the eigenvalue $a$ with multiplicity $1$, the eigenvectors being as before.
Then, since $\mathbf{b}$ is an eigenvector of $\mathbf{S}$ for the eigenvalue $a$, we have
\begin{equation}
\mathbf{A}\mathbf{b}=\mathbf{Q}\mathbf{S}\mathbf{b}=\mathbf{Q}(\mathbf{S}\mathbf{b})=\mathbf{Q}(a\mathbf{b})=a\mathbf{Q}\mathbf{b}.
\nonumber
\end{equation}
Since we also have $\mathbf{A}\mathbf{b}=a\mathbf{c}$, this implies
\begin{equation}
a\mathbf{Q}\mathbf{b}=a\mathbf{c} \Rightarrow \mathbf{Q}\mathbf{b}=\mathbf{c}.
\nonumber
\end{equation}
It follows that 
\begin{equation}
\mathbf{G} = \begin{bmatrix}
  a & \mathbf{b}^\top\\ 
  \mathbf{c} & \mathbf{A}
\end{bmatrix}=\begin{bmatrix}
  a & \mathbf{b}^\top\\ 
  \mathbf{Q}\mathbf{b} & \mathbf{Q}\mathbf{S}
\end{bmatrix}=\begin{bmatrix}
  1 & 0\\ 
  0 & \mathbf{Q}
\end{bmatrix}\begin{bmatrix}
  \sqrt{\Vert\mathbf{b}\Vert^2+1} & \mathbf{b}^\top\\ 
  \mathbf{b} & \sqrt{\mathbf{I}+\mathbf{b}\mathbf{b}^\top}
\end{bmatrix}.
\nonumber
\end{equation}
Based on Theorem~\ref{thm-1}, we can further parameterize the Euclidean orthogonal matrix $\mathbf{Q}\in\mathbb{R}^{(k-1)\times(k-1)}$ as the composition of $k-1$ Euclidean Householder reflections:
\begin{equation}
\mathbf{G} = \begin{bmatrix}
  1 & 0\\ 
  0 & \mathrm{Orth}(\mathbf{U},\mathbf{1})
\end{bmatrix}\begin{bmatrix}
  \sqrt{\Vert\mathbf{b}\Vert^2+1} & \mathbf{b}^\top\\ 
  \mathbf{b} & \sqrt{\mathbf{I}+\mathbf{b}\mathbf{b}^\top}
\end{bmatrix}.
\nonumber
\end{equation}
\end{proof}

\section{Proof of Claim~\ref{pro-ab}}
\label{proof-pro-4}
\begin{proof}
Our proposed GoldE framework models each relation $r$ as a generalized orthogonal transformation between entities. 
To facilitate presentation, we denote the $r$-specific generalized orthogonal matrix derived from Equation (\ref{mix-orth}) as $\mathbf{G}_r$:
\begin{align}
{\mathbf{G}}_r^{(i)} = 
  \begin{cases}
    \mathrm{Orth}_\mathbb{P}(\mathbf{U}_r^{(i)}, \mathbf{p}_r^{(i)}), & \text{for } \mathbb{X}_i=\mathbb{P} \\[1ex]
    \mathrm{Orth}_\mathbb{Q}(\mathbf{U}_r^{(i)},\mathbf{b}_r^{(i)}), & \text{for } \mathbb{X}_i=\mathbb{Q}
  \end{cases}
  .\nonumber
\end{align}
Due to the orthogonality of the relation matrix and the geometric unification of the embedding space, GoldE is capable of simultaneously capturing the crucial logical patterns and the inherent topological structures in KGs.
For simplicity, we omit the indices $(i)$ of component spaces in the following proofs. 

\textbf{Symmetry/antisymmetry pattern:} A relation $r$ is symmetric (antisymmetric) if $\forall{x, y}$
\begin{equation}
    r(x,y)\Rightarrow r(y,x) \quad (r(x,y)\Rightarrow \neg r(y,x)). \nonumber
\end{equation}
If $r(x,y)$ and $r(y,x)$ hold, we have
\begin{equation}
\mathbf{e}_y=\mathbf{G}_r\mathbf{e}_x \land \mathbf{e}_x=\mathbf{G}_r\mathbf{e}_y \Rightarrow \mathbf{G}_r\mathbf{G}_r=\mathbf{I}.\nonumber
\end{equation}
Otherwise, if $r(x,y)$ and $\neg r(y,x)$ hold, we have
\begin{equation}
\mathbf{e}_y=\mathbf{G}_r\mathbf{e}_x \land \mathbf{e}_x \neq \mathbf{G}_r\mathbf{e}_y \Rightarrow \mathbf{G}_r\mathbf{G}_r\neq\mathbf{I}.\nonumber
\end{equation}

\textbf{Inversion pattern:} A relation $r_1$ is inverse to relation $r_2$ if $\forall{x, y}$
\begin{equation}
    r_2(x,y)\Rightarrow r_1(y,x) \nonumber.               
\end{equation}
If $r_1(x,y)$ and $r_2(y,x)$ hold, we have
\begin{equation}
\mathbf{e}_y=\mathbf{G}_{r_1}\mathbf{e}_x \land \mathbf{e}_x=\mathbf{G}_{r_2}\mathbf{e}_y \Rightarrow \mathbf{G}_{r_1}=\mathbf{G}_{r_2}^\top.\nonumber
\end{equation}

\textbf{Composition pattern:} A relation $r_1$ is composed of relation $r_2$ and relation $r_3$ if $\forall{x, y, z}$
\begin{equation}
    r_2(x,y)\land r_3(y,z)\Rightarrow r_1(x,z). \nonumber
\end{equation}
If $r_1(x,z), r_2(x,y)$ and $r_3(y,z)$ hold, we have
\begin{equation}
\mathbf{e}_z=\mathbf{G}_{r_1}\mathbf{e}_x \land \mathbf{e}_y=\mathbf{G}_{r_2}\mathbf{e}_x \land \mathbf{e}_z=\mathbf{G}_{r_3}\mathbf{e}_y \Rightarrow \mathbf{G}_{r_1}=\mathbf{G}_{r_3}\mathbf{G}_{r_2}.\nonumber
\end{equation}

\textbf{Cyclical and hierarchical structures:} The fidelity of representations arises from the correspondence between the structure of the data and the geometry of the space~\cite{mix1}.
Our GoldE is endowed with a product manifold composed of multiple model spaces, enabling improved representations by better matching the geometry of the embedding space to the heterogeneous structures of KGs.
On the one hand, the orthogonal relation transformations in the elliptic component spaces are naturally suitable for capturing the cyclical structures~\cite{sph1,sph2}.
On the other hand, the hyperbolic orthogonal transformation in Equation (\ref{new_hyp_orth}) inherently encodes the  hierarchical structures---the Euclidean orthogonal matrix models the transformation between entities at the same level of hierarchies, and the hyperbolic boost matrix models the transformation between entities at different levels of hierarchies~\cite{hyp_hierachy}.
\end{proof}

\section{Efficient Computation}
\label{efficient}
The training time cost of GoldE primarily stems from the elliptic and hyperbolic orthogonal relation transformations in Equation (\ref{ell_orth}) and (\ref{new_hyp_orth}), where $k$ matrix-vector multiplications incur the time complexity of $O(k^3)$.
We show that these matrix-vector multiplications can be entirely replaced by vector-vector operations to achieve efficient computation.

For the elliptic orthogonal relation transformations in Equation (\ref{ell_orth}), the entity embedding $\mathbf{e}_{h,r}$ iteratively performs $k$ elliptic Householder reflections.
Each iteration can be achieved by vector-vector operations as follows:
\begin{equation}
\begin{split}
\label{fast_orth}
    \mathbf{e}_{h,r}^{i+1}&=\mathbf{H}(\mathbf{U}_r[i],\mathbf{p}_r)\mathbf{e}_{h,r}^i 
    = (\mathbf{I}-2\frac{\mathbf{U}_r[i]\mathbf{U}_r[i]^\top\mathrm{diag}(\mathbf{p}_r)}{\mathbf{U}_r[i]^\top\mathrm{diag}(\mathbf{p}_r)\mathbf{U}_r[i]})\mathbf{e}_{h,r}^i =\mathbf{e}_{h,r}^i-2\frac{\braket{\mathbf{U}_r[i],\mathbf{e}_{h,r}^i}_{\mathbf{p}_r}}{\braket{\mathbf{U}_r[i],\mathbf{U}_r[i]}_{\mathbf{p}_r}}\mathbf{U}_r[i],
\end{split}
\end{equation}
where $\mathbf{e}_{h,r}^1$ is the initial entity embedding $\mathbf{e}_{h,r}$.
In this way, the time complexity of Equation (\ref{ell_orth}) is reduced to $O(k^2)$.

For the hyperbolic orthogonal relation transformations in Equation (\ref{new_hyp_orth}), the entity embedding $\mathbf{e}_{h,r}$ is transformed by a hyperbolic boost matrix and a Euclidean orthogonal matrix.
We first replace the hyperbolic boost matrix with its equivalent form~\cite{equivalent_boost} to avoid calculating the matrix square root, and then accelerate the transformation via block matrix multiplication.
Formally, the hyperbolic boost transformation can be achieved by vector-vector operations as follows:
\begin{equation}
\begin{split}
    {\renewcommand{\arraycolsep}{0pt}\begin{bmatrix}
  \sqrt{\Vert \mathbf{b}_r \Vert^2+1} & \mathbf{b}_r^\top\\[3pt] 
  \mathbf{b}_r & \sqrt{\mathbf{I}+\mathbf{b}_r\mathbf{b}_r^\top}
\end{bmatrix}}\mathbf{e}_{h,r}&={\renewcommand{\arraycolsep}{3pt}\begin{bmatrix}
  \gamma & -\gamma\mathbf{v}_r^\top\\[3pt]
  -\gamma\mathbf{v}_r & \mathbf{I}+\frac{\gamma^2}{1+\gamma}\mathbf{v}_r\mathbf{v}_r^\top
\end{bmatrix}}{\renewcommand{\arraycolsep}{0pt}\begin{bmatrix}
  e_{h,r}^1\\[3pt]
  \mathbf{e}_{h,r}^{2:k}
\end{bmatrix}} = {\renewcommand{\arraycolsep}{0pt}\begin{bmatrix}
  \gamma e_{h,r}^1 - \gamma\braket{\mathbf{v}_r, \mathbf{e}_{h,r}^{2:k}}_\mathbf{1}\\[3pt] 
  -\gamma e_{h,r}^1\mathbf{v}_r + \mathbf{e}_{h,r}^{2:k} + \frac{\gamma^2\braket{\mathbf{v}_r, \mathbf{e}_{h,r}^{2:k}}_\mathbf{1}}{1+\gamma}\mathbf{v}_r
\end{bmatrix}},
\end{split}
\end{equation}
where $\mathbf{v}_r\in\mathbb{R}^{k-1}$, $\Vert\mathbf{v}_r\Vert<1$, $\gamma=\frac{1}{\sqrt{1-\Vert\mathbf{v}_r\Vert^2}}$, $e_{h,r}^1$ and $\mathbf{e}_{h,r}^{2:k}$ are the first and the remaining $k-1$ entries of $e_{h,r}$, respectively.
The subsequent Euclidean orthogonal transformation can also be replaced by vector-vector operations, which is analogous to Equation (\ref{fast_orth}).
Therefore, the time complexity of Equation (\ref{new_hyp_orth}) is also reduced to $O(k^2)$.

Table \ref{time} shows the convergence time required for training the models on three standard benchmarks. 
We select RotatE (for its simplicity) and HousE (for its advanced performance) as the baselines.
By using the efficient computation, our GoldE framework costs comparable training time to these two models. 
Combined with the link prediction results in Table \ref{WN18RR-and-FB15k237}, one can see that the GoldE is capable of achieving superior effectiveness without sacrificing the efficiency.
\begin{table}[h]
\vspace{-2mm}
\caption{Training time of RotatE, HousE and our GoldE on three datasets.}
\label{time}
\vspace{0.5mm}
\begin{center}
\begin{small}
\begin{tabular}{cccc}
\toprule
Model   & WN18RR & FB15k-237  & YAGO3-10 \\
\midrule
RotatE  & 4h         &  6h         & 10h         \\
HousE & 2h       &  3h         & 11h         \\
GoldE   & 2h       & 5h        &  13h      \\
\bottomrule
\end{tabular}
\end{small}
\end{center}
\vspace{-4mm}
\end{table}

\section{Datasets}
\label{Datasets}

Table~\ref{data_statistics} summarizes the detailed statistics of three benchmark datasets.

WN18RR \cite{wn18rr} and FB15k-237 \cite{fb15k237} datasets are subsets of WN18 \cite{TransE} and FB15k \cite{TransE} respectively with inverse relations removed to avoid test leakage.
WN18 is extracted from WordNet \cite{miller1995wordnet},  a database featuring lexical relations between words. FB15k is extracted from Freebase \cite{bollacker2008freebase}, a large-scale KG containing general knowledge facts.

YAGO3-10 is a subset of YAGO3 \cite{yago3-10}, containing 123,182 entities and 37 relations. Most of the triples in YAGO3-10 are descriptive attributes of people, such as citizenship, gender, profession and marital status.
\begin{table}[h]
\vspace{-2mm}
\caption{Statistics of five standard benchmarks.}
\label{data_statistics}
\vspace{1.0mm}
\begin{center}
\begin{small}
\resizebox{0.55\columnwidth}{!}{\begin{tabular}{cccccc}
\toprule
Dataset & \#Entity & \#Relation & \#Training & \#Validation & \#Test \\
\midrule
WN18RR    & 40,943 & 11 & 86,835 & 3,034 & 3,134 \\
FB15k-237    & 14,541 & 237 & 272,115 & 17,535 & 20,466 \\
YAGO3-10    & 123,182 & 37 & 1,079,040 & 5,000 & 5,000 \\
\bottomrule
\end{tabular}}
\end{small}
\end{center}
\vskip -0.1in
\end{table}

\section{Implementation Details}
\label{imp}
\textbf{Evaluation:} We follow the filtered ranking protocol~\cite{TransE} for evaluation. For a test triple $(h,r,t)$, we rank it against all negative triplets $(h,r,t')$ or $(h',r,t)$ that do not appear in the knowledge graph. We report mean rank (MR), mean reciprocal rank (MRR) and HITS at N (H@N) as performance metrics.

\textbf{Fair comparisons:} To ensure fair comparisons, we control the total number of parameters of GoldE to be similar to the baselines as shown in Table~\ref{params}.
Specifically, we follow \cite{house} to fix the embedding size $k$ of each entity as $800$, $600$, $1000$ on WN18RR, FB15k-237 and YAGO3-10 datasets, respectively.
\begin{table}[h]
\vspace{-4mm}
\caption{Comparison of the number of parameters. The results of baselines are taken from the original papers. } 
\label{params}
\vspace{0.5mm}
\begin{center}
\begin{small}
\resizebox{0.5\columnwidth}{!}{
\begin{tabular}{cccccc}
\toprule
Dataset  & RotatE & Rotate3D & DualE & HousE  & GoldE  \\
\midrule
WN18RR     & 40.95M & 61.44M & 32.76M & 32.57M & 31.23M  \\
FB15k-237    & 29.32M & 44.57M & 11.64M  & 12.13M & 15.93M  \\
YAGO3-10    & 123.18M & - & -  & 122.91M & 118.52M    \\
\bottomrule
\end{tabular}}
\end{small}
\end{center}
\vspace{-4mm}
\end{table}

\textbf{Selection of Product Manifold:} For reducing the number of hyperparameters, we set identical dimensions for component spaces of the same geometric type to describe the product space as $\mathbb{D}^k=\bigtimes_{i=1}^{m_\mathbb{P}}\mathbb{P}_{i}^{k_\mathbb{P}}\times \bigtimes_{j=1}^{m_\mathbb{Q}}\mathbb{Q}_{j}^{k_\mathbb{Q}}$, where $m_\mathbb{P}$ and $m_\mathbb{Q}$ denote the number of elliptic and hyperbolic component spaces, $k_\mathbb{P}$ and $k_\mathbb{Q}$ denote the dimension of elliptic and hyperbolic component spaces.
We omit $\mathbb{E}$ since the Euclidean orthogonal parameterization is the special case of our proposed elliptic orthogonal parameterization according to Claim~\ref{pro1}.
Considering that the embedding size $k$ is fixed to ensure fair comparisons, we can define the product manifold $\mathbb{D}^k$ with only three hyperparameters: the embedding size of one partition $k_*=m_\mathbb{P}k_\mathbb{P}$, the number of elliptic component spaces $m_\mathbb{P}$, and the number of hyperbolic component spaces $m_\mathbb{Q}$.
Note that the selected product manifold may not be the optimal case due to the simplified dimension configuration.
One potential direction is to design a signature\footnote[6]{The signature of product manifold refers to the number of component spaces of each geometric type and their dimensions.} estimator~\cite{mix1} for KGs, and we leave this as the future work.

\textbf{Hyperparameters:} We use Adam~\cite{kingma2014adam} as the optimizer and fine-tune the hyperparameters on the validation dataset.
The hyperparameters are tuned by the random search~\cite{bergstra2012random}, including batch size $b$, self-adversarial sampling temperature $\alpha$, fixed margin $\gamma$, learning rate $lr$, dimension $k_*$, number of elliptic component spaces $m_\mathbb{P}$, and number of hyperbolic component spaces $m_\mathbb{Q}$. The hyperparameter search space is shown in Table \ref{Hyper_search}.
\begin{table}[h]
\vspace{-4mm}
\caption{Hyperparameter search space.}
\label{Hyper_search}
\vspace{0.5mm}
\begin{center}
\begin{small}
\resizebox{0.43\columnwidth}{!}{
\begin{tabular}{ccc}
\toprule
Hyperparameter  & Search Space & Type  \\
\midrule
$b$  & $\{500, 800, 1000, 1500, 2000\}$ & Choice\\
$\alpha$  & $[0.5, 2.0]$ & Range\\
$\gamma$  & $\{6, 8, 10, 12, 16, 20, 24, 28\}$ & Choice\\
$lr$  & $\{0.0001, 0.0005, 0.001, 0.003\}$ & Choice\\
$k_*$  & $\{0, 200, 400, 600, 800\}$ & Choice\\
$m_\mathbb{P},m_\mathbb{Q}$ & $\{2,4,8,10,15,20,25,30\}$ & Choice \\
\bottomrule
\end{tabular}}
\end{small}
\end{center}
\vspace{-5mm}
\end{table}

\section{More Results of GoldE with Varying Embedding Size}
\label{app_curve}
Table \ref{full-low-setting} shows the full results of GoldE with embedding size $k$ equal to $32$.
Such low-dimensional setting follows \cite{RotH}.
One can observe that our GoldE framework consistently outperforms the baselines on all metrics across the three datasets, demonstrating the superior modeling capability of our designed universal orthogonal parameterization.
\vspace{-4mm}
\begin{table}[h]
\caption{Full results of GoldE with embedding size $k$ equal to $32$. 
}
\label{full-low-setting}
\vspace{0.5mm}
\begin{center}
\begin{small}
\resizebox{0.7\columnwidth}{!}{
\begin{tabular}{lcccccccccccc}
\toprule
         & \multicolumn{4}{c}{WN18RR} & \multicolumn{4}{c}{FB15k-237} & \multicolumn{4}{c}{YAGO3-10} \\ \cmidrule(r){2-5} \cmidrule(r){6-9} \cmidrule(r){10-13}
Model & MRR   & H@1 & H@3       & H@10        & MRR      & H@1 & H@3     & H@10 & MRR    & H@1 & H@3       & H@10          \\ 
\midrule
RotatE    & .387 & .330   & .417      & .491        & .290  & .208   & .316         &     .458    & - & -   & -   & -   \\
QuatE    & .421   & .396 & .430      & .467        & .293  & .212 & .320            &     .460    & -  & -   & -   & -   \\
MuRE  & .458   & .421 & .471      & .525        & .313  & .226       & .340     & .489        & .283  & .187 & .317  & .478       \\
MuRP  & .465  & .420 & .484       & .544        & .323  & .235 & .353            & \underline{.501}        & .230  & .150 & .247  & .392       \\
RefH & .447 & .408 & .464         & .518        & .312   & .224 & .342           & .489        & .381  & .302 & .415  & .530      \\
RotH & \underline{.472}   & \underline{.428}  & \underline{.490}    & \underline{.553}        & .314    & .223  & .346        & .497        & .393 & .307 & .435   & .559      \\
AttH & .466   & .419 & .484      & .551        & \underline{.324}     & \underline{.236} & \underline{.354}         & \underline{.501}        & \underline{.397}  & \underline{.310} & \underline{.437}  & \underline{.566}      \\
\midrule
GoldE    & \textbf{.482}  & \textbf{.447}    & \textbf{.493}       & \textbf{.558}        & \textbf{.335}   & \textbf{.246}   & \textbf{.360}        & \textbf{.506}        & \textbf{.420}     & \textbf{.330}     & \textbf{.465}   & \textbf{.586}      \\
\bottomrule
\end{tabular}}
\end{small}
\end{center}
\end{table}

To delve deeper into the impact of different embedding sizes, we further conduct experiments for GoldE with varying values of $k$.
The parameters in all cases are tuned, and the results are averaged over 5 runs with different random initializations.
Figure \ref{final-k} exhibits the MRR performance curves on WN18RR.
It reveals that GoldE consistently surpasses the baselines across a broad range of embedding sizes, comprehensively confirming the effectiveness of our proposal.
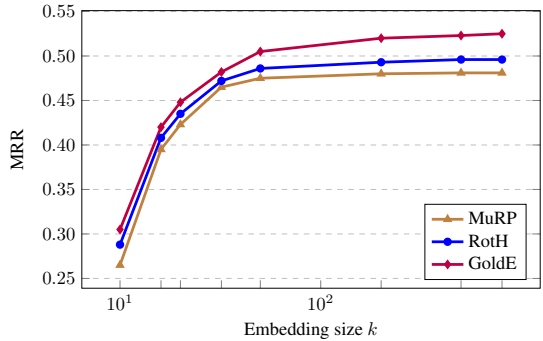
\begin{figure}[t!]
\vspace{5mm}
\centering
\begin{tikzpicture}[font=\normalsize, scale=0.7]
                \begin{axis}[
                    width=0.6\textwidth, 
                    height=0.4\textwidth,
                    legend cell align={left},
                    legend style={nodes={scale=1.0, transform shape}},
                    xlabel={Embedding size $k$},
                    xtick pos=left,
                    tick label style={font=\normalsize},
                    ylabel style={font=\normalsize},
                    ylabel={MRR},
                    xmode=log,
                    xtick={10, 16, 20, 32, 50, 100, 200, 500, 800},
                    xticklabels={$10^1$, , , , ,$10^2$},
                    ytick={0.25, 0.30,0.35,0.40,0.45,0.50,0.55},
                    yticklabels={$0.25$,$0.30$,$0.35$,$0.40$,$0.45$,$0.50$,$0.55$},
                    legend pos=south east,
                    ymajorgrids=true,
                    grid style=dashed
                ]
                \addplot[
                    color=brown,
                    solid,
                    mark options={solid},
                    mark=triangle*,
                    line width=1.5pt,
                    mark size=1.5pt
                    ]
                    coordinates {
                    (10, 0.265)
                    (16, 0.395)
                    (20, 0.423)
                    (32, 0.465)
                    (50, 0.475)
                    (200, 0.480)
                    (500, 0.481)
                    (800, 0.481)
                    };
                    \addlegendentry{MuRP}
                \addplot[
                    color=blue,
                    solid,
                    mark options={solid},
                    mark=*,
                    line width=1.5pt,
                    mark size=1.5pt
                    ]
                    coordinates {
                    (10, 0.288)
                    (16, 0.408)
                    (20, 0.435)
                    (32, 0.472)
                    (50, 0.486)
                    (200, 0.493)
                    (500, 0.496)
                    (800, 0.496)
                    };
                    \addlegendentry{RotH}
                \addplot[
                    color=purple,
                    solid,
                    mark options={solid},
                    mark=diamond*,
                    line width=1.5pt,
                    mark size=1.5pt
                    ]
                    coordinates {
                    (10, 0.305)
                    (16, 0.420)
                    (20, 0.448)
                    (32, 0.482)
                    (50, 0.505)
                    (200, 0.520)
                    (500, 0.523)
                    (800, 0.525)
                    };
                    \addlegendentry{GoldE}
                \end{axis}
                \end{tikzpicture}
    \vspace{-3.0mm}
    \caption{MRR results of the models with embedding size $k\in\{10,16,20,32,50,200,500,800\}$ on WN18RR.}
    \label{final-k}
    \vspace{-5mm}
\end{figure}

\section{Limitations and Future Work}
The proposed GoldE framework follows the traditional learning paradigm of KGE~\cite{VLP}. Specifically, GoldE (as well as existing KGE approaches) models each training triple independently to implicitly capture the multi-hop logical patterns. However, the single-triple constraints inevitably have biases during optimization. These biases will accumulate with the increase of relation hops and make the perceived logical patterns unreliable, thereby hindering the generalization ability. A possible solution is to combine our framework with a path/graph encoder~\cite{VLP, zhang2023att} to simultaneously process multiple structurally-dependent training triples. Furthermore, GoldE purely accesses the graph structures for symbolic inference. We plan to incorporate the textual attributes~\cite{li2017ppne, yan2023benchmark, zhao2023glem} of entities and relations in our future work to enhance the quality of learned representations. In addition, we would also like to investigate other theories and methodologies~\cite{yang2022continuous, liu2023study, li2024ripple, hu2024graphflow+, feng2024corporate} to unveil their potential significance in the realm of knowledge mining~\cite{rui2022knowledge}.


\end{document}